\documentclass[twoside,11pt]{article}
\usepackage{jmlr2e}
\usepackage{url}
\newcommand{\cH}{\mathcal{H}}
\newcommand{\corref}[1]{Corollary~\ref{#1}}
\newcommand{\badh}{\cH_{\otimes}}

\usepackage{times,amsmath,amsfonts,amssymb,epsfig,stmaryrd}
\usepackage{psfrag}
\usepackage{hyphenat}

\newtheorem{cor}[theorem]{Corollary}

\renewcommand{\eqref}[1]{Equation~(\ref{#1})}
\newcommand{\figref}[1]{Figure~\ref{#1}}
        
\newcommand{\secref}[1]{Section~\ref{#1}}
\newcommand{\thmref}[1]{Theorem~\ref{#1}}
\newcommand{\lemref}[1]{Lemma~\ref{#1}}
\newcommand{\defref}[1]{Definition~\ref{#1}}

\newcommand{\propref}[1]{Proposition~\ref{#1}}
\newcommand{\appref}[1]{Appendix~\ref{#1}}

\renewcommand{\P}{\mathbb{P}}
\newcommand{\E}{\mathbb{E}}
\newcommand{\reals}{\mathbb{R}}

\newcommand{\diag}{\mathrm{diag}}
\newcommand{\conv}{\mathrm{conv}}

\newcommand{\half}{{\frac12}}
\newcommand{\lambdamin}{\lambda_{\min}}
\newcommand{\lambdamax}{\lambda_{\max}}

\newcommand{\cX}{\mathcal{X}}

\newcommand{\cF}{\mathcal{F}}
\newcommand{\cA}{\mathcal{A}}

\newcommand{\cD}{\mathcal{D}}
\newcommand{\cW}{\mathcal{W}}
\newcommand{\cN}{\mathcal{N}}
\newcommand{\cG}{\mathcal{G}}
\newcommand{\cC}{\mathcal{C}}
\newcommand{\cR}{\mathcal{R}}
\newcommand{\cZ}{\mathcal{Z}}

\DeclareMathOperator*{\argmin}{argmin}

\newcommand{\loss}{\ell}

\newcommand{\binm}{\{\pm 1\}^m}
\newcommand{\binlab}{\{\pm 1\}}
\newcommand{\sign}{\textrm{sign}}

\newcommand{\norm}[1]{\|#1\|}
\newcommand{\dotprod}[1]{\langle #1 \rangle}

\newcommand{\ceil}[1]{{\lceil #1 \rceil}}

\newcommand{\floor}[1]{\lfloor #1\rfloor}
\newcommand{\ball}{\mathbb{B}^d_1}

\DeclareMathOperator{\trace}{trace}

\newcommand{\fullkgname}{margin-adapted dimension}

\newcommand{\rmom}{\rho}

\newcommand{\ramp}{\mathrm{ramp}}
\newcommand{\rampf}{\textsc{ramp}}
\newcommand{\chop}[1]{\llbracket {#1} \rrbracket}
\newcommand{\dfamily}{\cD^\textrm{sg}} 
\newcommand{\mt}[1]{\mathbb{#1}}

\jmlrheading{14}{2013}{2119-2149}{7/12; Revised 3/13}{7/13}{Sivan Sabato, Nati Srebro and Naftali Tishby}

\ShortHeadings{Distribution-Dependent Sample Complexity of Large Margin Learning}{Sabato, Srebro and Tishby}
\firstpageno{2119}

\title{Distribution-Dependent Sample Complexity \\ of Large Margin Learning}
\author{\name{Sivan Sabato} \email{sivan.sabato@microsoft.com} \\
\addr Microsoft Research New England\\
1 Memorial Drive\\
Cambridge, MA 02142, USA
\AND
\name{Nathan Srebro} \email{nati@ttic.edu}\\
\addr{Toyota Technological Institute at Chicago\\
6045 S. Kenwood Ave.\\
Chicago, IL 60637, USA}
\AND
\name{Naftali Tishby} \email{tishby@cs.huji.ac.il} \\
\addr The Rachel and Selim Benin School of Computer Science and Engineering\\
The Hebrew University\\
Jerusalem 91904, Israel
}

\editor{John Shawe-Taylor}

\begin{document}
\maketitle

\begin{abstract}%
  We obtain a tight distribution-specific characterization of the sample
  complexity of large-margin classification with $L_2$ regularization: We
  introduce the \emph{\mbox{\fullkgname}}, which is a simple function of the
  second order statistics of the data distribution, and show
  distribution-specific upper {\em and} lower bounds on the sample complexity,
  both governed by the \fullkgname\ of the data distribution. The upper bounds
  are universal, and the lower bounds hold for the rich family of sub-Gaussian
  distributions with independent features. We conclude that this new quantity tightly characterizes the true sample complexity of large-margin classification.  To prove the lower bound,
  we develop several new tools of independent interest. These include new connections
  between shattering and hardness of learning, new properties of shattering with linear classifiers, 
  and a new lower bound on the smallest eigenvalue of a random Gram matrix generated by sub-Gaussian variables. Our results can be used to quantitatively compare large margin learning to other learning rules, and to improve the effectiveness of methods that use sample complexity bounds, such as active learning.
\end{abstract}


\section{Introduction}

In this paper we pursue a tight characterization of the sample complexity of
learning a classifier, under a particular data distribution, and using a
particular learning rule. 

Most learning theory work focuses on providing sample-complexity upper bounds which hold for a large class of distributions. For instance,
standard distribution-free VC-dimension analysis shows that if one uses the Empirical Risk Minimization (ERM) learning rule, then the sample complexity of learning a
classifier from a hypothesis class with VC-dimension $d$ is at most $O\left(\frac{d}{\epsilon^2}\right)$, where  $\epsilon$ is the maximal excess classification error \citep{VapnikCh71,AnthonyBa99}. Such upper bounds can be useful for understanding the positive aspects of a
learning rule.  However, it is difficult to understand the deficiencies of a
learning rule, or to compare between different rules, based on upper bounds
alone. This is because it is possible, and is often the case, that the actual number of samples 
required to get a low error, for a given data distribution using a given learning rule, is much lower than the sample-complexity upper bound. As a simple example, suppose that the support of a given distribution is restricted to a subset of the domain. If the VC-dimension of the hypothesis class, when restricted to this subset, is smaller than $d$, then learning with respect to this distribution will require less examples than the upper bound predicts.

Of course, some sample complexity upper bounds are known to be tight or to have
an almost-matching lower bound. For instance, the VC-dimension upper bound is tight \citep{VapnikCh74}. This means that there exists
\emph{some} data distribution in the class covered by the upper bound, for
which this bound cannot be improved. Such a tightness result shows that there cannot be a better upper bound
that holds for this entire class of distributions. But it does not imply that
the upper bound characterizes the true sample complexity for every 
\emph{specific} distribution in the class.

The goal of this paper is to identify a simple quantity, which is a function of
the distribution, that {\em does} precisely characterize the sample complexity
of learning this distribution under a specific learning rule. We focus on the important hypothesis class of linear classifiers, and on the popular 
rule of \emph{margin-error-minimization} (MEM). Under this learning rule, a learner 
must always select a linear classifier that minimizes the margin-error on the input
sample.

The VC-dimension of the class of homogeneous linear classifiers in $\reals^d$ is $d$ \citep{Dudley78}. This implies a sample complexity upper bound of $O\left(\frac{d}{\epsilon^2}\right)$ using any MEM algorithm, where $\epsilon$ is the excess error relative
to the optimal margin error.\footnote{This upper bound can be derived analogously to the result for ERM algorithms with $\epsilon$ being the excess classification error. It can also be concluded from our analysis in \thmref{thm:upperbound} below.}
 We also have
that the sample complexity of any MEM algorithm is at most
$O\big(\frac{B^2}{\gamma^2\epsilon^2}\big)$, where $B^2$ is the average squared norm of the data and $\gamma$ is the size of the margin \citep{BartlettMe02}.
Both of these upper bounds are tight. For instance, there exists a distribution with 
an average squared norm of $B^2$ that requires as many as $C\cdot\frac{B^2}{\gamma^2\epsilon^2}$ examples to learn, for some universal constant $C$ \citep[see, e.g.,][]{AnthonyBa99}. However, the VC-dimension upper bound indicates, for instance, that if a distribution induces a large average norm but is supported by a low-dimensional sub-space, then the true number of examples required to reach a low error is much smaller. Thus, neither of 
these upper bounds fully describes the sample complexity of MEM for a \emph{specific} distribution.

We obtain
a tight distribution-specific characterization of the sample complexity of
large-margin learning for a rich class of distributions. 
We present a new quantity, termed the
\emph{\fullkgname}, and use it to provide a tighter distribution-dependent upper
bound, and a matching distribution-dependent lower bound for MEM. The upper bound is universal, 
and the lower bound holds for a rich class of distributions with independent features.

The \fullkgname\ refines both the dimension and the average norm of the data distribution, and
can be easily calculated from the covariance matrix and the mean of the distribution.
We denote this quantity, for a margin of $\gamma$, by $k_\gamma$. Our 
sample-complexity upper bound shows that
$\tilde{O}(\frac{k_\gamma}{\epsilon^2})$ examples suffice in order to learn any
distribution with a \fullkgname\ of $k_\gamma$ using a MEM algorithm with margin $\gamma$.  
We further show that for every distribution in a rich
family of `light tailed' distributions---specifically, product distributions of
sub-Gaussian random variables---the number of samples required for learning by
minimizing the margin error is at least $\Omega(k_{\gamma})$. 

Denote by $m(\epsilon,\gamma,D)$ the number of examples required to achieve an excess error of no more than $\epsilon$ relative to the best possible $\gamma$-margin error for a specific distribution $D$, using a MEM algorithm. 
Our main result shows the following matching distribution-specific upper and lower bounds on the sample complexity of MEM:
\begin{equation}\label{eq:mainresult}
 \Omega(k_\gamma(D)) \leq m(\epsilon,\gamma,D) \leq
 \tilde{O}\left(\frac{k_{\gamma}(D)}{\epsilon^2}\right).
\end{equation}

Our tight characterization, and in particular the distribution-specific lower
bound on the sample complexity that we establish, can be used to compare
large-margin ($L_2$ regularized) learning to other learning rules.  We provide
two such examples: we use our lower bound to rigorously establish a sample
complexity gap between $L_1$ and $L_2$ regularization previously studied in
\cite{Ng04}, and to show a large gap between discriminative and generative
learning on a Gaussian-mixture distribution. The tight bounds can also be used for active learning algorithms in which sample-complexity bounds are used to decide on the next label to query.

In this paper we focus only on large margin classification.  But
in order to obtain the distribution-specific lower bound, we develop
new tools that we believe can be useful for obtaining lower
bounds also for other learning rules. We provide several new results which we use to derive our main results. These include:
\begin{itemize}
\item Linking the fat-shattering of a sample with non-negligible probability to a difficulty of learning using MEM. 
\item Showing that for a convex hypothesis class, fat-shattering is equivalent to shattering with exact margins.
\item Linking the fat-shattering of a set of vectors with the eigenvalues of the Gram matrix of the vectors.
\item Providing a new lower bound for the smallest eigenvalue of a random Gram matrix generated by sub-Gaussian variables. This bound extends previous results in analysis of random matrices.
\end{itemize}

\subsection{Paper Structure} We discuss related work on sample-complexity upper bounds in \secref{sec:related}. We present the problem setting and notation in
\secref{sec:definitions}, and provide some necessary preliminaries in \secref{sec:prelim}. We then introduce the \fullkgname\ in
\secref{sec:marginadapted}.  The sample-complexity upper bound is proved in
\secref{sec:upper}. We prove the lower bound in \secref{sec:lower}. In
\secref{sec:limitations} we show that any non-trivial sample-complexity lower
bound for more general distributions must employ properties other than the covariance matrix of the distribution. We summarize and discuss implication in \secref{sec:conclusions}. Proofs omitted from the text are provided in \appref{app:proofs}

\section{Related Work} \label{sec:related}

As mentioned above, most work on ``sample complexity lower bounds'' is directed at proving
that under some set of assumptions, there exists a data distribution
for which one needs at least a certain number of examples to learn
with required error and confidence
\citep[for instance][]{AntosLu98,EhrenfeuchtHaKeVa88,GentileHe98}.  This type of a
lower bound does not, however, indicate much on the sample complexity
of other distributions under the same set of assumptions.

For distribution-specific lower bounds, the classical analysis of
\citet[Theorem 16.6]{Vapnik95} provides not only sufficient but
also necessary conditions for the learnability of a hypothesis class
with respect to a specific distribution.  The essential condition is
that the metric entropy of the hypothesis class with respect to
the distribution be sub-linear in the limit of an infinite sample
size.  In some sense, this criterion can be seen as providing a
``lower bound'' on learnability for a specific distribution.  However,
we are interested in finite-sample convergence rates, and would like
those to depend on simple properties of the distribution.  The
asymptotic arguments involved in Vapnik's general learnability claim
do not lend themselves easily to such analysis.

\citet{BenedekIt91} show that if the distribution is
known to the learner, a specific hypothesis class is learnable if and
only if there is a finite $\epsilon$-cover of this hypothesis class
with respect to the distribution. 
\citet{Ben-DavidLuPa08} consider a similar setting, and prove
sample complexity lower bounds for learning with any data
distribution, for some binary hypothesis classes on the real line. \citet{VayatisAz99} provide distribution-specific sample complexity
upper bounds for hypothesis classes with a limited VC-dimension, as a
function of how balanced the hypotheses are with respect to the
considered distributions. These bounds are not tight for all distributions, thus they also do not fully characterize the distribution-specific sample complexity.

As can be seen in \eqref{eq:mainresult}, we do not tightly characterize the dependence of the sample complexity on the
desired error \citep[as done, for example, in ][]{SteinwartSc07}, thus our bounds are not tight for asymptotically small error
levels.  Our results are most significant if the desired error level is a constant
well below chance but bounded away from zero.  This is in contrast to classical statistical
asymptotics that are also typically tight, but are valid only for
very small $\epsilon$.  As was recently shown by \citet{LiangSr10}, the sample
complexity for very small $\epsilon$ (in the classical
statistical asymptotic regime) depends on quantities that can be very different from those that control
the sample complexity for moderate error rates, which are more relevant for machine learning.

\section{Problem Setting and Definitions}\label{sec:definitions}
Consider a domain $\cX$, and let $D$ be a distribution over $\cX \times \{\pm 1\}$.  We denote by $D_X$
the marginal distribution of $D$ on $\cX$. 
The misclassification error of a classifier $h:\cX \rightarrow \reals$ on a distribution $D$ is
\[
\loss_0(h,D) \triangleq \P_{(X,Y)\sim D}[Y\cdot h(X) \leq 0].
\]  
The margin error of a classifier $w$ with respect to a margin
$\gamma > 0$ on $D$ is 
\[
\loss_\gamma(h,D) \triangleq \P_{(X,Y)\sim D}[Y\cdot h(X) \leq \gamma].
\]
For a given hypothesis class $\cH \subseteq \binlab^\cX$, the best achievable margin error on $D$ is 
\[
\loss^*_\gamma(\cH, D) \triangleq \inf_{h \in \cH}\loss_\gamma(h,D).
\]
We usually write simply $\loss^*_\gamma(D)$ since $\cH$ is clear from context.

A labeled sample is a (multi-)set $S = \{(x_i,y_i)\}_{i=1}^m \subseteq \cX \times \binlab$. Given $S$, we denote the set of its examples without their labels by $S_X \triangleq \{x_1,\ldots,x_m\}$. We use $S$ also to refer to the uniform distribution over the elements in $S$. Thus the misclassification error of $h:\cX \rightarrow \binlab$ on $S$ 
is 
\[
\loss(h,S) \triangleq \frac{1}{m}|\{i \mid y_i \cdot h(x_i) \leq 0\}|,
\]
and the $\gamma$-margin error on $S$ is \[
\loss_\gamma(h,S) \triangleq \frac{1}{m}|\{i \mid y_i \cdot h(x_i) \leq \gamma\}|.
\]

A learning algorithm is a function $\cA:\cup_{m=1}^\infty (\cX \times \binlab)^m \rightarrow \reals^\cX$, that receives a training set as input, and returns a function for classifying objects in $\cX$ into real values. The high-probability loss of an algorithm $\cA$ with respect to samples of size $m$, a distribution $D$ and a confidence parameter $\delta \in (0,1)$ is 
\[
\loss(\cA,D,m,\delta) = \inf\{\epsilon \geq 0 \mid \P_{S \sim D^m}[\loss(\cA(S), D) \geq \epsilon] \leq \delta\}.
\] 

In this work we investigate the sample complexity of learning using
margin-error minimization (MEM). The relevant class of algorithms is defined
as follows.
\begin{definition}
  An \emph{margin-error minimization (MEM) algorithm} $\cA$ maps a margin parameter $\gamma > 0$ to a learning algorithm $\cA_\gamma$, such that 
 \[
 \forall S \subseteq \cX \times \binlab, \quad\cA_\gamma(S) \in \argmin_{h \in \cH} \loss_\gamma(h,S).
 \]
\end{definition}

The distribution-specific sample complexity for MEM algorithms is the sample size required to guarantee low excess error for the given distribution. Formally, we have the following definition.
\begin{definition}[Distribution-specific sample complexity]
Fix a hypothesis class $\cH \subseteq \binlab^\cX$.
For $\gamma >0$, $\epsilon,\delta \in
[0,1]$, and a distribution $D$, the \emph{distribution-specific sample complexity}, denoted by  $m(\epsilon,\gamma,D,\delta)$, is the minimal sample size such that for any MEM
algorithm $\cA$, and for any $m \geq m(\epsilon,\gamma,D,\delta)$,
\[
\loss_0(\cA_\gamma,D,m,\delta) - \loss^*_\gamma(D) \leq \epsilon.
\]
\end{definition}
Note that we require that
\emph{all} possible MEM algorithms do well on the given distribution. This is because
we are interested in the MEM strategy in general, and thus we study the guarantees that can be provided regardless of any specific MEM implementation.
We sometimes omit $\delta$ and write simply $m(\epsilon, \gamma, D)$, to indicate that $\delta$ is assumed to be some fixed small constant.

In this work we focus on linear classifiers. For simplicity of notation, we assume a Euclidean space $\reals^d$ for some integer $d$, although the results can be easily extended to any separable Hilbert space.
For a real vector $x$, $\norm{x}$ stands for the Euclidean norm.
For a real matrix $\mt{X}$, $\norm{\mt{X}}$ stands for the Euclidean operator norm.

Denote the unit ball in $\reals^d$ by $\ball \triangleq \{w\in \reals^d \mid \norm{w} \leq 1\}$.
We consider the hypothesis class of homogeneous linear separators, $\cW = \{x \mapsto \dotprod{x,w} \mid w \in \ball\}$. We often slightly abuse notation by using $w$ to denote the mapping $x \mapsto \dotprod{x,w}$. 

We often represent sets of vectors in $\reals^d$ using matrices. We say that
$\mt{X} \in \reals^{m\times d}$ is the matrix of a set $\{x_1,\ldots,x_m\}\subseteq
\reals^d$ if the rows in the matrix are exactly the vectors in the set. For
uniqueness, one may assume that the rows of $\mt{X}$ are sorted according to an arbitrary fixed full
order on vectors in $\reals^d$. For a PSD matrix $\mt{X}$ denote the largest 
eigenvalue of $\mt{X}$ by $\lambdamax(\mt{X})$ and the smallest eigenvalue by $\lambdamin(\mt{X})$.

We use the $O$-notation as follows: $O(f(z))$ stands for $C_1+C_2f(z)$
for some constants $C_1,C_2\geq 0$. $\Omega(f(z))$ stands for $C_2f(z)-C_1$
for some constants $C_1,C_2\geq 0$. $\widetilde{O}(f(z))$ stands for $f(z)p(\ln(z)) + C$ for some polynomial $p(\cdot)$ and some constant $C > 0$.

\section{Preliminaries}\label{sec:prelim}

As mentioned above, for the hypothesis class of linear classifiers $\cW$, one can derive a sample-complexity upper bound of the form $O(B^2/\gamma^2 \epsilon^2)$, where $B^2 = \E_{X \sim D}[\norm{X}^2]$ and $\epsilon$ is the excess error relative to the $\gamma$-margin loss. This can be achieved as follows \citep{BartlettMe02}. 
Let $\cZ$ be some domain. The empirical Rademacher complexity of a class of functions $\cF \subseteq \reals^\cZ$  with respect to a set $S = \{z_i\}_{i\in[m]} \subseteq \cZ$
is 
\begin{equation*}
\cR(\cF, S) = \frac{1}{m}\E_\sigma[|\sup_{f \in \cF} \sum_{i\in[m]}\sigma_i f(z_i)|],
\end{equation*}
where $\sigma = (\sigma_1,\ldots,\sigma_m)$ are $m$ independent uniform $\{\pm1\}$-valued variables.
The average Rademacher complexity of $\cF$ with respect to a distribution $D$ over $\cZ$ and a sample size $m$ is 
\begin{equation*}
\cR_m(\cF, D) = \E_{S \sim D^m}[\cR(\cF, S)].
\end{equation*}

Assume a hypothesis class $\cH \subseteq \reals^\cX$ and a loss function $\loss:\binlab\times \reals \rightarrow \reals$. For a hypothesis $h \in \cH$, we introduce the function $h_\loss:\cX \times \binlab \rightarrow \reals$,
 defined by $h_\loss(x,y) = \loss(y,h(x))$. We further define the function class $\cH_\loss = \{ h_\loss \mid h \in \cH\} \subseteq \reals^{\cX \times \binlab}$.

Assume that the range of $\cH_\loss$ is in $[0,1]$. For any $\delta \in (0,1)$, with probability of $1-\delta$ over the draw of samples $S \subseteq \cX \times \binlab$ of size $m$ according to $D$, every $h \in \cH$ satisfies \citep{BartlettMe02}
\begin{equation}\label{eq:rademacherind}
\loss(h, D) \leq \loss(h, S) + 2\cR_m(\cH_\loss, D) + \sqrt{\frac{8\ln(2/\delta)}{m}}.
\end{equation}

To get the desired upper bound for linear classifiers we use the \emph{ramp loss}, which is defined as follows.
For a number $r$, denote $\chop{r} \triangleq \min(\max(r,0),1)$.
The $\gamma$-ramp-loss of a labeled example $(x,y) \in \reals^d \times \binlab$ 
with respect to a linear classifier $w \in \ball$
is $\ramp_\gamma(w,x,y) = \chop{1-y\dotprod{w,x}/\gamma}$. 
Let $\ramp_\gamma(w,D) = \E_{(X,Y)\sim D}[\ramp_\gamma(w,X,Y)]$,
and denote the class of ramp-loss functions by
\[
\rampf_\gamma = \{(x,y) \mapsto \ramp_\gamma(w,x,y) \mid w \in \ball\}.
\]
The ramp-loss is upper-bounded by the margin loss and lower-bounded by the misclassification error. Therefore, the following result can be shown. 
\begin{proposition}\label{prop:ramp}
For any MEM algorithm $\cA$, we have
\begin{equation}\label{eq:propstatement}
\loss_0(\cA_\gamma, D, m, \delta) \leq \loss^*_\gamma(\cH,D) + 2\cR_m(\rampf_\gamma, D) + \sqrt{\frac{14\ln(2/\delta)}{m}}.
\end{equation}
\end{proposition}
We give the proof in \appref{app:radramp} for completeness.
Since the $\gamma$-ramp loss is $1/\gamma$ Lipschitz, it follows from \cite{BartlettMe02} that  
\[
\cR_m(\rampf_\gamma, D) \leq \sqrt{\frac{B^2}{\gamma^2 m}}.
\]
Combining this with Proposition~\ref{prop:ramp} we can conclude a sample complexity upper bound of $O(B^2/\gamma^2 \epsilon^2)$.

In addition to the Rademacher complexity, we will also use the classic notions of \emph{fat-shattering} \citep{KearnsSc94} and \emph{pseudo-shattering} \citep{Pollard84}, defined as follows.
\begin{definition}\label{def:shattered}
  Let $\cF$ be a set of functions $f:\cX \rightarrow \reals$, and let $\gamma >
  0$.  The set $\{x_1,\ldots,x_m\} \subseteq \cX$ is
  \emph{$\gamma$-shattered} by $\cF$ with the witness $r \in \reals^m$ if for all $y \in \binm$ there is an $f \in \cF$ such that
  $\forall i\in [m],\:y[i](f(x_i)-r[i]) \geq \gamma$.  
\end{definition}
The $\gamma$-shattering dimension of a hypothesis class is the size of the largest set that is $\gamma$-shattered by this class.
We say that a set is \emph{$\gamma$-shattered at the origin} if it is 
$\gamma$-shattered with the zero vector as a witness.

\begin{definition}\label{def:pseudo}
  Let $\cF$ be a set of functions $f:\cX \rightarrow \reals$, and let $\gamma >
  0$.  The set $\{x_1,\ldots,x_m\} \subseteq \cX$ is
  \emph{pseudo-shattered} by $\cF$ with the witness $r \in \reals^m$ if for all $y \in \binm$ there is an $f \in \cF$ such that
  $\forall i\in [m],\:y[i](f(x_i)-r[i]) > 0$.  
\end{definition}
The pseudo-dimension of a hypothesis class is the size of the largest set that is pseudo-shattered by this class.

\section{The Margin-Adapted Dimension}\label{sec:marginadapted}

When considering learning of linear classifiers using MEM, the dimension-based upper bound and the norm-based upper bound are both tight in the worst-case sense,
that is, they are the best bounds that rely only on the dimensionality or
only on the norm respectively. Nonetheless, neither is tight in a
distribution-specific sense: If the average norm is unbounded while the
dimension is small, then there can be an arbitrarily large gap between the true
distribution-dependent sample complexity and the bound that depends on the average norm. If the
converse holds, that is, the dimension is arbitrarily large while the
average-norm is bounded, then the dimensionality bound is loose.

Seeking a tight distribution-specific analysis, one simple approach to
 tighten these bounds is to consider their minimum, which is proportional to 
$\min(d,B^2/\gamma^2)$. Trivially, this is an upper
bound on the sample complexity as well.  However, this simple combination is also
not tight: Consider a distribution in which there are a few directions with
very high variance, but the combined variance in all other
directions is small (see \figref{fig:illustration}).  We will show that in such situations the sample
complexity is characterized not by the minimum of dimension and norm, but by the sum of the number of high-variance dimensions and the average squared norm in the other directions. This behavior is captured by the \emph{\fullkgname} which we presently define, using the following auxiliary definition.

\begin{figure}[b]
\centering
\includegraphics[width = 0.7\textwidth]{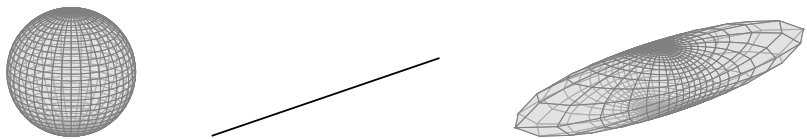}
\caption{Illustrating covariance matrix ellipsoids. left: norm bound is tight; middle: dimension bound is tight; right: neither bound is tight.}
\label{fig:illustration}
\end{figure}

\begin{definition}\label{def:limited}
Let $b>0$ and let $k$ be a positive integer. A distribution $D_X$ over $\reals^d$ is \emph{$(b,k)$-limited} if there exists a sub-space $V \subseteq \reals^d$ of dimension $d-k$ such that 
$
\E_{X\sim D_X}[\norm{\mt{O}_V\cdot X}^2] \leq b,
$
where $\mt{O}_V$ is an orthogonal projection onto $V$. 
\end{definition}

\begin{definition}[\fullkgname]
The \emph{\fullkgname} of a distribution $D_X$,
denoted by $k_\gamma(D_X)$, is the minimum
$k$ such that the distribution is $(\gamma^2k,k)$-limited.
\end{definition}

We sometimes drop the argument of $k_\gamma$ when it is clear from context.  It
is easy to see that for any distribution $D_X$ over $\reals^d$,
$k_{\gamma}(D_X) \leq\min(d,\E[\norm{X}^2]/\gamma^2)$. Moreover, $k_\gamma$ can
be much smaller than this minimum. For example, consider a random vector $X \in
\reals^{1001}$ with mean zero and statistically independent coordinates, such
that the variance of the first coordinate is $1000$, and the variance in each
remaining coordinate is $0.001$. We have $k_1=1$ but $d = \E[\norm{X}^2] =
1001$.

$k_\gamma(D_X)$ can be calculated from the uncentered covariance matrix $\E_{X\sim D_X}[XX^T]$ as follows: Let
$\lambda_1 \geq \lambda_2 \geq \cdots \lambda_d\geq 0$ be the eigenvalues of this matrix. Then 
\begin{equation}\label{eq:kgammamin}
k_{\gamma} = \min \{ k \mid \sum_{i=k+1}^d \lambda_i \leq
\gamma^2 k\}.
\end{equation}
A quantity similar to this definition of $k_\gamma$ was studied previously in \cite{Bousquet02}. The eigenvalues of the \emph{empirical} covariance matrix were used to provide sample complexity bounds, for instance 
in \cite{ScholkopfShSmWi99}. However, $k_\gamma$ generates a different type of bound, 
since it is defined based on the eigenvalues of the distribution and not of the sample. We will see that for small finite samples, the latter can be quite different from the former. 

Finally, note that while we define the \fullkgname\ for a finite-dimensional space for ease of notation, the same definition carries over to an infinite-dimensional Hilbert space. Moreover, $k_\gamma$ can be finite even if some of the eigenvalues $\lambda_i$ are infinite, implying a distribution with unbounded covariance.

\section{A Distribution-Dependent Upper Bound}\label{sec:upper}

In this section we prove an upper bound on the sample complexity of learning
with MEM, using the \fullkgname. We do this by providing a tighter upper bound for the Rademacher complexity of $\rampf_\gamma$. 
We bound $\cR_m(\rampf_\gamma,D)$ for any $(B^2,k)$-limited distribution $D_X$, using $L_2$ covering numbers, defined as follows.

Let $(\cX, \norm{\cdot}_\circ)$ be a normed space.
An $\eta$-covering of a set $\cF \subseteq \cX$ with respect to the norm $\norm{\cdot}_\circ$
is a set $\cC \subseteq \cX$ such that 
for any $f\in \cF$ there exists a $g \in \cC$ such that 
$\norm{f -g}_\circ \leq \eta.$
The covering-number for given $\eta> 0$, $\cF$ and $\circ$ is the size of the smallest such $\eta$-covering, and is denoted by $\cN(\eta, \cF, \circ)$.
Let $S = \{x_1,\ldots,x_m\} \subseteq \reals^d$. For a function $f:\reals^d \rightarrow \reals$, the $L_2(S)$ norm of $f$ is $\norm{f}_{L_2(S)} = \sqrt{\E_{X \sim S}[f(X)^2]}$.
Thus, we consider covering-numbers of the form $\cN(\eta, \rampf_\gamma, L_2(S))$.

The empirical Rademacher complexity of a function class can be bounded by the $L_2$ covering numbers of the same function class as follows \citep[Lemma 3.7]{Mendelson02}:
Let $\epsilon_i = 2^{-i}$. Then 
\begin{equation}\label{eq:mendelson}
\sqrt{m}\cR(\rampf_\gamma,S) \leq C\sum_{i\in[N]}\epsilon_{i-1}\sqrt{\ln\cN(\epsilon_{i}, \rampf_\gamma, L_2(S))} + 2\epsilon_{N}\sqrt{m}.
\end{equation}

To bound the covering number of $\rampf_\gamma$, we will restate the functions in $\rampf_\gamma$ as sums of two functions, each selected from a function class with bounded complexity. 
The first function class will be bounded because of the norm bound on the subspace $V$ used in \defref{def:limited}, and the second function class will have a bounded pseudo-dimension. However, the second function class will depend on the choice of the first function in the sum. Therefore, we require the following lemma, which provides an upper bound on such sums of functions.   
We use the notion of a \emph{Hausdorff distance} between two sets $\cG_1,\cG_2\subseteq \cX$, defined as $\Delta_H(\cG_1,\cG_2) = \sup_{g_1 \in \cG_1} \inf_{g_2 \in \cG_2} \norm{g_1 - g_2}_\circ$.

\begin{lemma}\label{lem:doublecover}
Let $(\cX, \norm{\cdot}_\circ)$ be a normed space.  Let $\cF \subseteq \cX$ be a set,
and let $\cG:\cX \rightarrow 2^{\cX}$ be a mapping from objects in $\cX$ to sets
of objects in $\cX$.  Assume that $\cG$ is $c$-Lipschitz with respect to the
Hausdorff distance on sets, that is assume that 
\[
\forall f_1,f_2\in \cX, \Delta_H(\cG(f_1),\cG(f_2)) \leq c\norm{f_1 - f_2}_\circ.
\]
Let $\cF_\cG = \{ f+g \mid f \in \cF, g \in \cG(f)\}$.
Then 
\[
\cN(\eta, \cF_\cG, \circ) \leq \cN(\eta/(2+c), \cF, \circ)\cdot
\sup_{f\in\cF}\cN(\eta/(2+c), \cG(f),\circ).
\]
\end{lemma}
\begin{proof}
For any set $A \subseteq \cX$, denote by $\cC_A$ a minimal $\eta$-covering for $A$ with respect to $\norm{\cdot}_\circ$, so that $|\cC_A| = \cN(\eta, A, \circ)$.
Let $f+g \in \cF_\cG$ such that $f \in \cF, g \in \cG(f)$.
There is a $\hat{f} \in \cC_\cF$ such that $\norm{f - \hat{f}}_\circ\leq \eta$.
In addition, by the Lipschitz assumption there is a $\tilde{g} \in \cG(\hat{f})$ such that $\norm{g - \tilde{g}}_\circ \leq c\norm{f - \hat{f}}_\circ \leq c\eta$. Lastly, there is a $\hat{g} \in \cC_{\cG(\hat{f})}$ such that $\norm{\tilde{g} - \hat{g}}_\circ \leq \eta$. 
Therefore 
\[
\norm{f + g - (\hat{f} + \hat{g})}_\circ \leq \norm{f - \hat{f}}_\circ + \norm{g- \tilde{g}}_\circ + \norm{\tilde{g} - \hat{g}}_\circ \leq (2+c)\eta.
\]
Thus the set $\{ f+g \mid f \in \cC_\cF, g \in \cC_{\cG(f)}\}$ is a $(2+c)\eta$ cover of $\cF_\cG$.
The size of this cover is at most $|\cC_\cF|\cdot\sup_{f \in \cF}|\cC_{\cG(f)}| \leq \cN(\eta, \cF, \circ)\cdot
\sup_{f\in\cF}\cN(\eta, \cG(f),\circ)$.
\end{proof}

The following lemma provides us with a useful class of mappings which are $1$-Lipschitz with respect to the Hausdorff distance, as required in \lemref{lem:doublecover}. The proof is provided in \appref{app:glipschitz}.
\begin{lemma}\label{lem:glipschitz}
Let $f:\cX \rightarrow \reals$ be a function and let $Z \subseteq \reals^\cX$ be a function class over some domain $\cX$.
Let $\cG:\reals^\cX \rightarrow 2^{\reals^\cX}$
be the mapping defined by 
\begin{equation}\label{eq:gf}
\cG(f) \triangleq\{ x \mapsto \chop{f(x) + z(x)} - f(x) \mid z \in Z\}.
\end{equation}
Then $\cG$ is $1$-Lipschitz with respect to the Hausdorff distance.
\end{lemma}
The function class induced by the mapping above preserves the pseudo-dimension of the original function class, as the following lemma shows. The proof is provided in \appref{app:lempseudo}.

\begin{lemma}\label{lem:pseudo}
Let $f:\cX \rightarrow \reals$ be a function and let $Z \subseteq \reals^\cX$ be a function class over some domain $\cX$.
Let $\cG(f)$ be defined as in \eqref{eq:gf}.
Then the pseudo-dimension of $\cG(f)$ is at most the pseudo-dimension of $Z$.
\end{lemma}

Equipped with these lemmas, we can now provide the new bound on the Rademacher complexity of $\rampf_\gamma$ in the following theorem. The subsequent corollary states the resulting sample-complexity upper bound for MEM, which depends on $k_\gamma$.
\begin{theorem}\label{thm:upperbound}
Let $D$ be a distribution over $\reals^d \times \binlab$, and assume $D_X$ is $(B^2,k)$-limited. 
Then
\[
\cR(\rampf_\gamma, D) \leq \sqrt{\frac{O(k + B^2/\gamma^2)\ln(m)}{m}}.
\]
\end{theorem}

\begin{proof}
In this proof all absolute constants are assumed to
be positive and are denoted by $C$ or $C_i$ for some integer $i$. 
Their values may change from line to line or even within the same line.

Consider the distribution $\tilde{D}$ which results from drawing $(X,Y) \sim D$ and emitting \mbox{$(Y\cdot X, 1)$}. It too is $(B^2,k)$-limited, and $\cR(\rampf_\gamma,D) = \cR(\rampf_\gamma,\tilde{D})$. 
Therefore, we assume without loss of generality that for all $(X,Y)$ drawn from $D$, $Y = 1$.
Accordingly, we henceforth omit the $y$ argument from $\ramp_\gamma(w,x,y)$ and write simply $\ramp_\gamma(w,x) \triangleq \ramp_\gamma(w,x,1)$.

Following \defref{def:limited}, Let $\mt{O}_V$ be an orthogonal projection onto a sub-space $V$ of dimension $d-k$ 
such that $\E_{X\sim D_X}[\norm{\mt{O}_V\cdot X}^2] \leq B^2$. Let $\bar{V}$ be the complementary sub-space to $V$. For a set $S = \{x_1,\ldots,x_m\} \subseteq \reals^d$, denote $B(S) = \sqrt{\frac{1}{m}\sum_{i\in[m]}\norm{\mt{O}_V\cdot X}^2}$. 

We would like to use \eqref{eq:mendelson} to bound the Rademacher complexity of $\rampf_\gamma$. Therefore, we will bound $\cN(\eta,\rampf_\gamma, L_2(S))$ for $\eta > 0$. Note that 
\[
\ramp_\gamma(w,x) = \chop{1-\dotprod{w,x}/\gamma}\allowbreak = 1 - \chop{\dotprod{w,x}/\gamma}.
\]
Shifting by a constant and negating do not change the covering number of a function class.
Therefore, $\cN(\eta,\rampf_\gamma, L_2(S))$ is equal to the covering number of $\{ x \mapsto \chop{\dotprod{w,x}/\gamma} \mid w \in \ball\}$.
Moreover, let 
\[
\rampf_\gamma' = \{ x \mapsto \chop{\dotprod{w_a + w_b,x}/\gamma} \mid
w_a \in \ball \cap V,\: w_b \in \bar{V}\}.
\]
Then $\{ x \mapsto \chop{\dotprod{w,x}/\gamma} \mid w \in \ball\} \subseteq \rampf_\gamma'$,
thus it suffices to bound $\cN(\eta,\rampf'_\gamma, L_2(S))$.
To do that, we show that $\rampf_\gamma'$ satisfies the assumptions of \lemref{lem:doublecover} for  the normed space $(\reals^{\reals^d},\norm{\cdot}_{L_2(S)})$.
Define 
\[
\cF = \{ x \mapsto \dotprod{w_a,x}/\gamma \mid w_a \in \ball\cap V\}.
\]
Let $\cG:\reals^{\reals^d} \rightarrow 2^{\reals^{\reals^d}}$ 
be the mapping defined by 
\[
\cG(f) \triangleq\{ x \mapsto \chop{f(x) + \dotprod{w_b,x}/\gamma} - f(x) \mid w_b \in \bar{V}\}.
\]
Clearly, $\cF_\cG = \{f + g \mid f \in \cF, g \in \cG(f)\} = \rampf_\gamma'$.
Furthermore, by \lemref{lem:glipschitz}, $\cG$ is $1$-Lipschitz with respect to the Hausdorff distance. 
Thus, by \lemref{lem:doublecover}
\begin{equation}\label{eq:multcovers}
\cN(\eta, \rampf_\gamma', L_2(S)) \leq \cN(\eta/3, \cF, L_2(S))\cdot
\sup_{f\in\cF}\cN(\eta/3, \cG(f),L_2(S)).
\end{equation}

We now proceed to bound the two covering numbers on the right hand side.
First, consider $\cN(\eta/3, \cG(f),L_2(S))$. By \lemref{lem:pseudo}, the
pseudo-dimension of $\cG(f)$ is the same as the pseudo-dimension of $\{x \mapsto \dotprod{w,x}/\gamma \mid w \in \bar{V}\}$, which is exactly $k$, the dimension of $\bar{V}$.
The $L_2$ covering number of $\cG(f)$ can be bounded by the pseudo-dimension of $\cG(f)$ as follows \citep[see, e.g.,][Theorem 3.1]{Bartlett06}:
\begin{equation}\label{eq:pseudocover}
\cN(\eta/3,\cG(f),L_2(S)) \leq C_1\left(\frac{C_2}{\eta^2}\right)^k.
\end{equation}
Second, consider $\cN(\eta/3, \cF, L_2(S))$. Sudakov's minoration theorem (\citealt{Sudakov71}, and see also
\citealp{LedouxTa91}, Theorem 3.18) states that for any $\eta > 0$
\[
\ln\cN(\eta, \cF, L_2(S)) \leq \frac{C}{m\eta^2}\E_{s}^2[\sup_{f \in \cF} \sum_{i\in[m]}s_i f(x_i)],
\]
where $s = (s_1,\ldots,s_m)$ are independent standard normal variables. The right-hand side can be bounded
as follows:
\begin{align*}
&\gamma \E_s[\sup_{f \in \cF}|\sum_{i=1}^m s_i f(x_i)|]
= \E_s[\sup_{w \in \ball\cap V}|\dotprod{w,\sum_{i=1}^m s_i x_i}|]\\
&\quad\leq \E_s[\norm{\sum_{i=1}^m s_i\mt{O}_V x_i}] 
\leq \sqrt{\E_s[\norm{\sum_{i=1}^m s_i\mt{O}_V x_i}^2]}
= \sqrt{\sum_{i\in[m]}\norm{\mt{O}_V x_i}^2} = \sqrt{m}B(S).
\end{align*}
Therefore $\ln \cN(\eta, \cF, L_2(S)) \leq C\frac{B^2(S)}{\gamma^2\eta^2}.$
Substituting this and \eqref{eq:pseudocover} for the right-hand side in \eqref{eq:multcovers}, and adjusting constants, we get
\[
\ln\cN(\eta, \rampf_\gamma, L_2(S)) \leq \ln\cN(\eta, \rampf_\gamma', L_2(S)) \leq C_1(1 + k\ln(\frac{C_2}{\eta})+\frac{B^2(S)}{\gamma^2\eta^2}),
\]
To finalize the proof, we plug this inequality into \eqref{eq:mendelson} to 
get 
\begin{align*}
&\sqrt{m}\cR(\rampf_\gamma, S) \leq C_1\sum_{i\in[N]}\epsilon_{i-1}\sqrt{1 + k\ln(C_2/\epsilon_i)+\frac{B^2(S)}{\gamma^2\epsilon_i^2}} + 2\epsilon_{N}\sqrt{m} \\
&\leq C_1\left(\sum_{i\in[N]}\epsilon_{i-1}\left(1 + \sqrt{k\ln(C_2/\epsilon_i)}+\sqrt{\frac{B^2(S)}{\gamma^2\epsilon_i^2}}\right)\right) + 2\epsilon_{N}\sqrt{m}\\
&= C_1\left(\sum_{i\in[N]}2^{-i+1} + \sqrt{k}\sum_{i\in[N]}2^{-i+1}\ln(C_2/2^{-i}) + \sum_{i\in[N]}\frac{B(S)}{\gamma}\right) + 
2^{-N+1}\sqrt{m} \\
&\leq C\left(1 + \sqrt{k} + \frac{B(S)\cdot N}{\gamma}\right) + 2^{-N+1}\sqrt{m}.
\end{align*}
In the last inequality we used the fact that $\sum_{i}i2^{-i+1} \leq 4$.
Setting $N = \ln(2m)$ we get
\begin{align*}
&\cR(\rampf_\gamma, S) \leq \frac{C}{\sqrt{m}}\left(1 + \sqrt{k} + \frac{B(S)\ln(2m)}{\gamma}\right).
\end{align*}

Taking expectation over both sides, and noting that $\E[B(S)]\leq \sqrt{\E[B^2(S)]} \leq B$, we get 
\[
\cR(\rampf_\gamma, S) \leq \frac{C}{\sqrt{m}}(1 + \sqrt{k} + \frac{B\ln(2m)}{\gamma}) 
\leq \sqrt{\frac{O(k + B^2\ln^2(2m)/\gamma^2)}{m}}.
\]

\end{proof}

\begin{cor}[Sample complexity upper bound]\label{cor:upperbound}
Let $D$ be a distribution over $\reals^d\times \{\pm1\}$. Then
\[
m(\epsilon,\gamma,D) \leq \tilde{O}\left(\frac{k_\gamma(D_X)}{\epsilon^2}\right).
\]
\end{cor}
\begin{proof}
By \propref{prop:ramp}, we have 
\[
\loss_0(\cA_\gamma, D, m, \delta) \leq \loss^*_\gamma(\cW,D) + 2\cR_m(\rampf_\gamma, D) + \sqrt{\frac{14\ln(2/\delta)}{m}}.
\]
By definition of $k_\gamma(D_X)$, $D_X$ is $(\gamma^2 k_\gamma, k_\gamma)$-limited. 
Therefore, by \thmref{thm:upperbound}, 
\[
\cR_m(\rampf_\gamma, D) \leq \sqrt{\frac{O(k_\gamma(D_X))\ln(m)}{m}}.
\]
We conclude that 
\[
\loss_0(\cA_\gamma, D, m, \delta) \leq \loss^*_\gamma(\cW,D) + \sqrt{\frac{O(k_\gamma(D_X)\ln(m) + \ln(1/\delta))}{m}}.
\]
Bounding the second right-hand term by $\epsilon$, we conclude that $m(\epsilon,\gamma,D) \leq \tilde{O}(k_\gamma/\epsilon^2)$.
\end{proof}

One should note that a similar upper bound can be obtained much more easily under a uniform upper bound on the eigenvalues of the uncentered covariance matrix.\footnote{This has been pointed out to us by an anonymous reviewer of this manuscript. An upper bound under sub-Gaussianity assumptions can be found in \cite{SabatoSrTi10b}.} However, such an upper bound would not capture the fact that a finite dimension implies a finite sample complexity, regardless of the size of the covariance.
If one wants to estimate the sample complexity, then large covariance matrix eigenvalues imply that more examples are required to estimate the covariance matrix from a sample. However, these examples need not be labeled. Moreover, estimating the covariance matrix is not necessary to achieve the sample complexity, since the upper bound holds for any margin-error minimization algorithm.

\section{A Distribution-Dependent Lower Bound}\label{sec:lower}

The new upper bound presented in \corref{cor:upperbound} can be tighter than both the norm-only and the dimension-only
upper bounds. But does the \fullkgname\ characterize the true sample
complexity of the distribution, or is it just another upper bound? To answer
this question, we first need tools for deriving sample complexity lower bounds.  \secref{sec:fatlower} relates fat-shattering with a lower bound on sample complexity. 
In \secref{sec:lowerbound} we use this result to relate the smallest
eigenvalue of a Gram-matrix to a lower bound on sample
complexity. In \secref{sec:subg} the family of sub-Gaussian product distributions is presented. We prove a sample-complexity
lower bound for this family in \secref{sec:lowerboundsg}.

\subsection{A Sample Complexity Lower Bound Based on Fat-Shattering}\label{sec:fatlower}

The ability to learn is closely related to the probability of a sample to be
shattered, as evident in Vapnik's formulations of learnability as a function of
the $\epsilon$-entropy \citep{Vapnik95}. It is well known that the maximal size
of a shattered set dictates a sample-complexity upper bound. In the theorem below, we show that for some hypothesis classes it also
implies a lower bound. 
The theorem states that if a sample drawn
from a data distribution is fat-shattered with a non-negligible probability,
then MEM can fail to learn a good classifier for this distribution.\footnote{In contrast, the average 
Rademacher complexity cannot be used to derive general lower bounds for MEM algorithms, since it is related to the rate of uniform convergence of the entire hypothesis class, while MEM algorithms choose low-error hypotheses \citep[see, e.g.,][]{BartlettBoMe05}.}
This holds not only for linear classifiers, but more generally for all \emph{symmetric} hypothesis classes. Given a domain $\cX$, we say that a hypothesis class $\cH \subseteq \reals^\cX$ is symmetric if for all $h \in \cH$, we have $-h \in \cH$ as well. 
This clearly holds for the class of linear classifiers $\cW$.

\begin{theorem}\label{thm:shatterednotlearned}
Let $\cX$ be some domain, and assume that $\cH \subseteq \reals^\cX$ is a symmetric hypothesis class.
Let $D$ be a distribution over $\cX\times \{\pm 1\}$.
If the probability of a sample of size $m$ drawn from $D_X^{m}$ to be $\gamma$-shattered at the origin
by $\cW$ is at least $\eta$, then $m(\epsilon,\gamma,D,\eta/2) \geq \floor{m/2}$ for all $\epsilon < 1/2 - \loss^*_\gamma(D)$.
\end{theorem}
\begin{proof}
Let $\epsilon \leq \frac12 - \loss^*_\gamma(D)$. We show a MEM algorithm $\cA$
such that
\[
\loss_0(\cA_\gamma,D,\floor{m/2},\eta/2) \geq \frac12 > \loss^*_\gamma(D) + \epsilon,
\]
thus proving the desired lower bound on $m(\epsilon,\gamma,D,\eta/2)$.

Assume for simplicity that $m$ is even (otherwise replace $m$ with $m-1$).
Consider two sets $S,\tilde{S} \subseteq \cX \times \binlab$, each of size $m/2$, 
such that $S_X\cup \tilde{S}_X$ is $\gamma$-shattered at the origin by $\cW$. 
Then there exists a hypothesis $h_1 \in \cH$ such that the following holds:
\begin{itemize}
\item For all $x \in S_X \cup \tilde{S}_X$, $|h_1(x)| \geq \gamma$.
\item For all $(x,y) \in S$, $\sign(h_1(x)) = y$.
\item For all $(x,y) \in \tilde{S}$, $\sign(h_1(x)) = -y$.
\end{itemize}
It follows that $\loss_\gamma(h_1,S) = 0$. In addition, let $h_2 = -h_1$. Then $\loss_\gamma(h_2,\tilde{S}) = 0$. Moreover, we have $h_2 \in \cH$ due to the symmetry of $\cH$. 
On each point in $\cX$, at least one of $h_1$ and $h_2$ predict the wrong sign. Thus $\loss_0(h_1,D) + \loss_0(h_2,D) \geq 1$. It follows that for at least one of $i \in \{1,2\}$, we have \mbox{$\loss_0(h_i, D) \geq \half$}. 
Denote the set of hypotheses with a high misclassification error by 
\[
\badh = \{ h\in \cH \mid \loss_0(h,D) \geq \half\}.
\]
We have just shown that if $S_X\cup \tilde{S}_X$ is $\gamma$-shattered by $\cW$ then at least one of the following holds: (1) $h_1 \in \badh \cap \argmin_{h\in \cH} \loss_\gamma(h,S)$ or (2) $h_2 \in \badh \cap \argmin_{h\in \cH} \loss_\gamma(h,\tilde{S})$.

Now, consider a MEM algorithm $\cA$ such that whenever possible, it returns a hypothesis from $\badh$. Formally, given the input sample $S$, if $\badh \cap \argmin_{h\in \cH} \loss_\gamma(h,S) \neq \emptyset$, then \linebreak[4] $\cA(S) \in \badh \cap \argmin_{h\in \cH} \loss_\gamma(h,S)$. It follows that 
\begin{align*}
&\P_{S \sim D^{m/2}}[\loss_0(\cA(S),D) \geq \tfrac12] \geq \P_{S \sim D^{m/2}}[\badh \cap \argmin_{h\in \cH} \loss_\gamma(h, S) \neq \emptyset]\\
&\quad= \half (\P_{S \sim D^{m/2}}[\badh \cap\argmin_{h\in \cH} \loss_\gamma(h, S)\neq \emptyset] + \P_{\tilde{S}\sim D^{m/2}}[\badh \cap\argmin_{h\in \cH} \loss_\gamma(h, \tilde{S}) \neq \emptyset]) \\
&\quad\geq \half (\P_{S,\tilde{S} \sim D^{m/2}}[\badh \cap\argmin_{h\in \cH} \loss_\gamma(h, S)\neq \emptyset \:\text{ OR }\: \badh \cap\argmin_{h\in \cH} \loss_\gamma(h, \tilde{S}) \neq \emptyset])\\
&\quad\geq \half \P_{S,\tilde{S} \sim D^{m/2}}[S_X\cup \tilde{S}_X \text{ is $\gamma$-shattered at the origin }].
\end{align*}
The last inequality follows from the argument above regarding $h_1$ and $h_2$.
The last expression is simply half the probability that a sample of size $m$ from $D_X$ is shattered. By assumption, this probability is at least $\eta$. 
Thus we conclude that $\P_{S \sim D^{m/2}}[\loss_0(\cA(S),D) \geq \half] \geq \eta/2.$
It follows that $\loss_0(\cA_\gamma,D,m/2,\eta/2) \geq \half$.
\end{proof}

As a side note, it is interesting to observe that \thmref{thm:shatterednotlearned} does not hold in general for non-symmetric hypothesis classes. For example, 
 assume that the domain is $\cX = [0,1]$, and the hypothesis class is the
set of all functions that label a finite number of points in $[0,1]$ by $+1$
and the rest by $-1$. Consider learning using MEM, when the distribution is uniform over $[0,1]$, and all the labels are $-1$.
For any $m > 0$ and $\gamma \in (0,1)$, a sample of size $m$ is $\gamma$-shattered at the origin with probability $1$.
However, any learning algorithm that returns a hypothesis from the hypothesis class will incur zero error on this distribution. Thus, shattering alone does not suffice to ensure that learning is hard.

\subsection{A Sample Complexity Lower Bound with Gram-Matrix Eigenvalues}\label{sec:lowerbound}

We now return to the case of homogeneous linear classifiers, and link high-probability fat-shattering to properties of the distribution. 
First, we present an equivalent and simpler
characterization of fat-shattering for linear classifiers. We then use it 
to provide a sufficient condition for the fat-shattering of a sample, based on the smallest eigenvalue of its Gram matrix.

\begin{theorem}\label{thm:shattercond}
  Let $\mt{X} \in \reals^{m\times d}$ be the matrix of a set of size $m$ in $\reals^d$. The set is
  $\gamma$-shattered at the origin by $\cW$ if and only if $\mt{X}\mt{X}^T$ is invertible and for all $y \in \binm$, $y^T (\mt{X}\mt{X}^T)^{-1} y  \leq \gamma^{-2}$.
\end{theorem}
To prove \thmref{thm:shattercond} we require two auxiliary lemmas.  The
first lemma, stated below, shows that for convex function classes, 
$\gamma$-shattering can be substituted with shattering with exact
$\gamma$-margins. 
\begin{lemma}\label{lem:exactmargin} 
Let $\cF\subseteq \reals^\cX$ be a class of functions,
and assume that $\cF$ is convex, that is 
\[
\forall f_1,f_2\in \cF, \forall \lambda \in [0,1],\quad
\lambda f_1 + (1-\lambda)f_2\in \cF. 
\]
If $S = \{x_1,\ldots,x_m\} \subseteq \cX$ is $\gamma$-shattered by $\cF$ with
witness $r \in \reals^m$, then for every $y \in \binm$ there is an $f \in \cF$
such that for all $i \in [m],\:y[i] (f(x_i) - r[i]) = \gamma$.
\end{lemma}
The proof of this lemma is provided in \appref{app:shattercond}.
The second lemma that we use allows converting the representation
of the Gram-matrix to a different feature space, while keeping the separation properties intact.
For a matrix $\mt{M}$, denote its pseudo-inverse by $\mt{M}^+$.
\begin{lemma}\label{lem:wtilde}
Let $\mt{X} \in \reals^{m\times d}$ be a matrix such that $\mt{X} \mt{X}^T$ is invertible, and let $\mt{Y}\in \reals^{m\times k}$ such that $\mt{X}\mt{X}^T = \mt{Y} \mt{Y}^T$.
Let $r \in \reals^m$ be some real vector.
If there exists a vector $\widetilde{w} \in \reals^k$ such that $\mt{Y} \widetilde{w} = r$, then
there exists a vector $w \in \reals^d$ such that $\mt{X} w = r \text{ and } \norm{w} = \norm{ \mt{Y}^T (\mt{Y}^T)^+ \widetilde{w}}  \leq \norm{\tilde{w}}$.
\end{lemma}
\begin{proof}
Denote $\mt{K} = \mt{X}\mt{X}^T = \mt{Y}\mt{Y}^T$.
Let $\mt{S} = \mt{Y}^T \mt{K}^{-1} \mt{X}$ and let $w = \mt{S}^T \widetilde{w}$. We have 
$\mt{X} w = \mt{X} \mt{S}^T \widetilde{w} = \mt{X} \mt{X}^T \mt{K}^{-1} \mt{Y}\widetilde{w} = \mt{Y} \widetilde{w} = r.$ In addition,
$
\| w \|^2 = w^T w = \widetilde{w}^T \mt{S} \mt{S}^T \widetilde{w}.
$
By definition of $\mt{S}$, 
\[
\mt{S} \mt{S}^T = \mt{Y}^T \mt{K}^{-1} \mt{X} \mt{X}^T \mt{K}^{-1} \mt{Y} = \mt{Y}^T \mt{K}^{-1} \mt{Y} = \mt{Y}^T (\mt{Y} \mt{Y}^T)^{-1} \mt{Y} = \mt{Y}^T (\mt{Y}^T)^{+}.
\]
Denote $\mt{O} = \mt{Y}^T (\mt{Y}^T)^{+}$. $\mt{O}$ is an orthogonal projection matrix: by the properties of the pseudo-inverse, $\mt{O} = \mt{O}^T$ and $\mt{O}^2 = \mt{O}$. Therefore
$
\norm{w}^2 = \widetilde{w}^T \mt{S}  \mt{S}^T \widetilde{w}  = \widetilde{w}^T \mt{O} \widetilde{w} = 
 \widetilde{w}^T \mt{O} \mt{O}^T \widetilde{w} = \| \mt{O} \widetilde{w} \|^2 \leq \norm{\widetilde{w}}^2.
$
\end{proof}

\begin{proof}[of \thmref{thm:shattercond}]
We prove the theorem for $1$-shattering. The case of $\gamma$-shattering follows by rescaling $X$ appropriately.
Let $\mt{X}\mt{X}^T = \mt{U} \Lambda \mt{U}^T$ be the SVD of $\mt{X}\mt{X}^T$, where $\mt{U}$ is an orthogonal matrix and $\Lambda$ is a diagonal matrix.
Let $\mt{Y} = \mt{U} \Lambda^\half$. We have $\mt{X}\mt{X}^T = \mt{Y} \mt{Y}^T$. We show that the specified conditions are sufficient and necessary for the shattering of the set:

\begin{enumerate}
\item 
Sufficient: If $\mt{X}\mt{X}^T$ is invertible, then $\Lambda$ is invertible, thus so is $\mt{Y}$. 
For any $y\in \binm$, Let $w_y = \mt{Y}^{-1} y$. 
Then $\mt{Y} w_y = y$. By \lemref{lem:wtilde}, 
there exists a separator $w$ such that $\mt{X}w = y$ and $\norm{w} \leq \norm{w_y} = \sqrt{y^T (\mt{Y}\mt{Y}^T)^{-1}y} = \sqrt{y^T (\mt{X}\mt{X}^T)^{-1}y} \leq 1$. 

\item Necessary: If $\mt{X}\mt{X}^T$ is not invertible then the vectors in $S$ are linearly dependent, thus $S$ cannot be shattered using linear separators \citep[see, e.g.,][]{Vapnik95}. The first condition is therefore necessary. Assume $S$ is $1$-shattered at the origin and show that the second condition necessarily holds. By \lemref{lem:exactmargin}, for all $y \in \binm$ there exists a $w_y \in \ball$ such that $\mt{X}w_y = y$.
Thus by \lemref{lem:wtilde} there exists a $\widetilde{w}_y$ such that $\mt{Y} \widetilde{w}_y = y$ and $\norm{\widetilde{w}_y} \leq \norm{w_y} \leq 1$.  $\mt{X}\mt{X}^T$ is invertible, thus so is $\mt{Y}$. Therefore $\widetilde{w}_y = \mt{Y}^{-1}y$. Thus $y^T(\mt{X}\mt{X}^T)^{-1}y = y^T(\mt{Y}\mt{Y}^T)^{-1}y = \norm{\widetilde{w}_y} \leq 1$.
\end{enumerate}
\end{proof}

We are now ready to provide a sufficient condition for fat-shattering based on the smallest eigenvalue of the Gram matrix. 
\begin{cor}\label{cor:lambdam}
Let $\mt{X} \in \reals^{m\times d}$ be the matrix of a set of size $m$ in $\reals^d$.  
If $\lambdamin(\mt{X}\mt{X}^T) \geq m\gamma^2$ then the set is $\gamma$-shattered at the origin by $\cW$.
\end{cor}
\begin{proof}
  If $\lambdamin(\mt{X}\mt{X}^T) \geq m\gamma^2$ then $\mt{X}\mt{X}^T$ is invertible and
  $\lambdamax((\mt{X}\mt{X}^T)^{-1})\leq (m\gamma^2)^{-1}$.  For any $y \in \binm$ we have
  $\norm{y}=\sqrt{m}$ and 
  \[
  y^T (\mt{X}\mt{X}^T)^{-1} y \leq \norm{y}^2
  \lambdamax((\mt{X}\mt{X}^T)^{-1}) \leq m(m\gamma^2)^{-1} = \gamma^{-2}.
  \]  
  By \thmref{thm:shattercond} the sample is $\gamma$-shattered at the origin.
\end{proof}

\corref{cor:lambdam} generalizes the
requirement of linear independence for shattering with no margin:  A set of vectors is shattered with no margin if the vectors are linearly independent, that is if $\lambdamin>0$. 
The corollary shows that for $\gamma$-fat-shattering, we can require instead $\lambdamin \geq m\gamma^2$. We can now conclude that if it is highly probable that the smallest eigenvalue
of the sample Gram matrix is large, then MEM might fail to learn
a good classifier for the given distribution. This is formulated in the following theorem. 
\begin{theorem}\label{thm:inductive}
  Let $D$ be a distribution over $\reals^d\times \{\pm 1\}$.
  Let $m > 0$ and let $\mt{X}$ be the matrix of a sample drawn from $D^{m}_X$. Let $\eta = \P[\lambdamin(\mt{X} \mt{X}^T) \geq m \gamma^2]$.
  Then for all $\epsilon < 1/2 - \loss^*_\gamma(D)$, $m(\epsilon,\gamma,D,\eta/2) \geq \floor{m/2}$.
\end{theorem}
The proof of the theorem is immediate by combining \thmref{thm:shatterednotlearned} and \corref{cor:lambdam}. 

\thmref{thm:inductive} generalizes the case of learning a linear separator
without a margin: If a sample of size $m$ is linearly independent with high
probability, then there is no hope of using $m/2$ points to predict the label
of the other points.  The theorem extends this observation to the case of
learning with a margin, by requiring a stronger condition than just linear
independence of the points in the sample.

Recall that our upper-bound on the sample complexity from
\secref{sec:upper} is $\tilde{O}(k_{\gamma})$.  We now define 
the family of sub-Gaussian product distributions, and show that for this family, the lower bound that can be deduced from \thmref{thm:inductive} is also linear in $k_{\gamma}$.

\subsection{Sub-Gaussian Distributions}\label{sec:subg}
In order to derive a lower bound on distribution\hyp{}specific sample complexity in
terms of the covariance of $X \sim D_X$, we must assume that $X$ is not too
heavy-tailed. This is because for any data distribution there exists another distribution which
is almost identical and has the same sample complexity, but has arbitrarily
large covariance values. This can be achieved by mixing the original
distribution with a tiny probability for drawing a vector with a huge norm.  We
thus restrict the discussion to multidimensional sub-Gaussian
distributions. This ensures light tails of the distribution in all directions,
while still allowing a rich family of distributions, as we presently see.
Sub-Gaussianity is defined for scalar random variables
as follows \citep[see, e.g.,][]{BuldyginKo98}. 
\begin{definition}[Sub-Gaussian random variables]
A random variable $X \in \reals$ is \emph{sub-Gaussian with moment $B$}, for $B \geq 0$, if 
\begin{equation*}
\forall t \in \reals, \quad \E[\exp(tX)]\leq \exp(t^2B^2 /2).
\end{equation*}
In this work we further say that $X$ is sub-Gaussian with
\emph{relative moment} $\rmom > 0$ if $X$ is sub-Gaussian with moment $\rho\sqrt{\E[X^2]}$, that is, 
\[
\forall t \in \reals, \quad \E[\exp(tX)]\leq \exp(t^2\rmom^2\E[X^2] /2).
\]
\end{definition}
Note that a sub-Gaussian variable with moment $B$ and relative moment $\rho$ is also sub-Gaussian with moment $B'$ and relative moment $\rho'$ for any $B' \geq B$ and $\rho' \geq \rho$. 

The family of sub-Gaussian distributions is quite extensive: For instance, it
includes any bounded, Gaussian, or Gaussian-mixture random variable with mean
zero.  Specifically, if $X$ is a mean-zero Gaussian random variable, $X \sim N(0, \sigma^2)$,
then $X$ is sub-Gaussian with relative moment $1$ and the inequalities in the definition above
hold with equality.  As another example, if $X$ is a
uniform random variable over $\{\pm b\}$ for some $b \geq 0$, then $X$ is sub-Gaussian with relative moment $1$, since
\begin{equation}\label{eq:bernoulli}
\E[\exp(tX)] = \half(\exp(tb) + \exp(-tb)) \leq \exp(t^2b^2/2) = \exp(t^2\E[X^2]/2).
\end{equation}
Let $\mt{B} \in \reals^{d\times d}$ be a symmetric PSD matrix. 
A random vector $X \in \reals^d$ is a \emph{sub-Gaussian random vector} with moment matrix $\mt{B}$ 
if for all $u \in \reals^d$, $\E[\exp(\dotprod{u,X})] \leq \exp(\dotprod{\mt{B}u,u}/2)$.
The following lemma provides a useful connection between the trace of the sub-Gaussian moment matrix and the moment-generating function of the squared norm of the random vector.
The proof is given in \appref{app:sgvecmgf}.
\begin{lemma}\label{lem:sgvecmgf}
Let $X \in \reals^d$ be a sub-Gaussian random vector 
with moment matrix $\mt{B}$. 
Then for all $t \in (0,\frac{1}{4\lambdamax(\mt{B})}]$, $\E[\exp(t \norm{X}^2)] \leq \exp(2t \cdot \trace(\mt{B})).$
\end{lemma}

Our lower bound holds for the family of sub-Gaussian product distributions, defined as follows.
\begin{definition}[Sub-Gaussian product distributions]\label{def:indsubg}
  A distribution $D_X$ over $\reals^d$ is a \emph{sub-Gaussian product
      distribution} with moment $B$ and relative moment $\rmom$ if there exists some
  orthonormal basis $a_1,\ldots,a_d \in \reals^d$, such that for $X \sim D_X$,
  $\dotprod{a_i, X}$ are independent sub-Gaussian random
  variables, each with moment $B$ and relative moment $\rmom$.
\end{definition} 
Note that a sub-Gaussian product distribution has mean zero, thus its 
covariance matrix is equal to its uncentered covariance matrix.
For any fixed $\rmom \geq 0$, we denote by $\dfamily_\rmom$ the family of all
sub-Gaussian product distributions with relative moment $\rmom$, in arbitrary
dimension. For instance, all multivariate Gaussian distributions and all
uniform distributions on the corners of a centered hyper-rectangle are in
$\dfamily_1$. All uniform distributions over a full centered hyper-rectangle are in
$\dfamily_{3/2}$. Note that if $\rmom_1 \leq \rmom_2$, $\dfamily_{\rmom_1}
\subseteq \dfamily_{\rmom_2}$.
 
 We will provide a lower bound for all distributions in $\dfamily_\rmom$. This lower
 bound is linear in the \fullkgname\ of the distribution, thus it matches the
 upper bound provided in \corref{cor:upperbound}. The constants in the lower
 bound depend only on the value of $\rmom$, which we regard as a
 constant.

\subsection{A Sample-Complexity Lower Bound for Sub-Gaussian Product Distributions}\label{sec:lowerboundsg}

As shown in \secref{sec:lowerbound}, to obtain a sample complexity lower bound
it suffices to have a lower bound on the value of the smallest eigenvalue of a
random Gram matrix.  The distribution of the smallest eigenvalue of a random
Gram matrix has been investigated under various assumptions. The cleanest
results are in the asymptotic case where the sample size and the dimension approach
infinity, the ratio between them approaches a constant, and the coordinates of
each example are identically distributed.

\begin{theorem}[{\citealt[Theorem 5.11]{BaiSi10}}]\label{thm:asym}
Let $\{\mt{X}_i\}_{i=1}^\infty$ be a series of matrices of sizes $m_i \times d_i$, whose entries are i.i.d.~random variables with mean zero, variance $\sigma^2$ and finite fourth moments. If $\lim_{i\rightarrow \infty}\frac{m_i}{d_i} = \beta < 1$, then
$\lim_{i\rightarrow \infty} \lambdamin(\frac{1}{d_i}\mt{X}_i\mt{X}_i^T) = \sigma^2(1-\sqrt{\beta})^2.$
\end{theorem}

This asymptotic limit can be used to approximate an asymptotic lower bound on $m(\epsilon,\gamma,D)$,
if $D_X$ is a product distribution of i.i.d.~random variables with mean zero, variance $\sigma^2$, and finite fourth moment.  Let $\mt{X} \in \reals^{m\times d}$ be the matrix of a sample of size $m$ drawn from $D_X$. We can find
$m = m_\circ$ such that $\lambda_{m_\circ}(\mt{X}\mt{X}^T) \approx \gamma^2m_\circ$, and use \thmref{thm:inductive} to conclude that $m(\epsilon,\gamma,D) \geq m_\circ/2$. If $d$ and $m$ are
large enough, we have by \thmref{thm:asym} that for $\mt{X}$ drawn from $D_X^m$:
\[
\lambdamin(\mt{X}\mt{X}^T) \approx d \sigma^2 (1-\sqrt{m/d})^2  = \sigma^2(\sqrt{d}-\sqrt{m})^2.
\]
Solving the equality $\sigma^2(\sqrt{d}-\sqrt{m_\circ})^2=m_\circ\gamma^2$ we get
$m_\circ = d/(1+\gamma/\sigma)^2$.  The \fullkgname\ for $D_X$ is $k_\gamma
\approx d/(1+\gamma^2/\sigma^2)$, thus $\tfrac{1}{2}
k_\gamma \leq
m_\circ \leq k_\gamma$.  In this case, then, the sample complexity lower bound is indeed
the same order as $k_\gamma$, which controls also the upper bound in \corref{cor:upperbound}.  However,
this is an asymptotic analysis, which holds for a highly limited set of distributions.
Moreover, since \thmref{thm:asym} holds asymptotically for each
distribution separately, we cannot use it to deduce a uniform finite-sample
lower bound for families of distributions.

For our analysis we require \emph{finite-sample} bounds for the
smallest eigenvalue of a random Gram-matrix.  
\citet{RudelsonVe09,RudelsonVe08} provide such
finite-sample lower bounds for distributions which are products 
of identically distributed sub-Gaussians.  In \thmref{thm:smallesteigwhpsg} below we
provide a new and more general result, which holds for any sub-Gaussian product distribution. The proof of \thmref{thm:smallesteigwhpsg}
   is provided in \appref{app:smallesteigwhpsg}. Combining \thmref{thm:smallesteigwhpsg} with \thmref{thm:inductive} above
   we prove the lower bound, stated in \thmref{thm:lowerboundsg} below.
\begin{theorem}\label{thm:smallesteigwhpsg}
For any $\rmom> 0$ and $\delta \in (0,1)$ there are $\beta > 0$ and \mbox{$C > 0$} such that the following holds.
For any $D_X \in \dfamily_\rmom$ with covariance matrix $\Sigma \leq I$,
and for any $m \leq \beta \cdot \trace(\Sigma) - C$,
if $\mt{X}$ is the $m\times d$ matrix of a sample drawn from $D_X^m$, then
\[
\P[\lambdamin(\mt{X} \mt{X}^T) \geq m] \geq \delta.
\]
\end{theorem}

\begin{theorem}[Sample complexity lower bound for distributions in $\dfamily_\rmom$]
\label{thm:lowerboundsg}
For any $\rmom >0$ there are constants $\beta > 0,C\geq 0$ such that for any $D$ with $D_X \in \dfamily_\rmom$, for any $\gamma > 0$ and for any $\epsilon < \frac{1}{2} - \loss^*_\gamma(D)$,
\[
m(\epsilon,\gamma,D,1/4) \geq \beta k_\gamma(D_X)-C.
\]
\end{theorem}
\begin{proof} 
Assume w.l.o.g. that the orthonormal basis $a_1,\ldots,a_d$ of independent sub-Gaussian 
directions of $D_X$, defined in \defref{def:indsubg}, is the natural basis $e_1,\ldots,e_d$. Define $\lambda_i = \E_{X\sim D_X}[X[i]^2]$,
and assume w.l.o.g. $\lambda_1 \geq \ldots \geq \lambda_d > 0$. 
Let $\mt{X}$ be the $m\times d$ matrix of a sample drawn from $D_X^m$. Fix $\delta \in (0,1)$, and let $\beta$ and $C$ be the constants for $\rmom$ and $\delta$ in \thmref{thm:smallesteigwhpsg}.
Throughout this proof we abbreviate $k_\gamma \triangleq k_\gamma(D_X)$.
Let $m \leq \beta (k_\gamma-1) - C$. 
We would like to use \thmref{thm:smallesteigwhpsg} to bound $\lambdamin(\mt{X}\mt{X}^T)$ with high probability, so that \thmref{thm:inductive} can be applied to get the desired lower bound. However, \thmref{thm:smallesteigwhpsg} holds only if $\Sigma \leq I$. Thus we split to two cases---one in which the dimensionality controls the lower bound, and one in which the norm controls it. The split is based on the value of $\lambda_{k_\gamma}$. 
\begin{itemize}
\item Case I: Assume $\lambda_{k_\gamma} \geq \gamma^2$. Then $\forall i\in[k_\gamma],\lambda_i \geq \gamma^2$.
By our assumptions on $D_X$, for all $i\in[d]$ the random variable $X[i]$ is sub-Gaussian
with relative moment $\rmom$. Consider the random variables $Z[i] = X[i]/\sqrt{\lambda_i}$ for $i \in [k_\gamma]$. $Z[i]$ is also sub-Gaussian with relative moment $\rmom$, and $\E[Z[i]^2] = 1$.
Consider the product distribution of $Z[1],\ldots,Z[k_\gamma]$,
and let $\Sigma'$ be its covariance matrix. We have $\Sigma' = I_{k_\gamma}$,
and $\trace(\Sigma') = k_\gamma$.
Let $\mt{Z}$ be
the matrix of a sample of size $m$ drawn from this distribution. By \thmref{thm:smallesteigwhpsg},
$\P[\lambdamin(\mt{Z} \mt{Z}^T)\geq m] \geq \delta$, which is equivalent to 
\[
\P[\lambdamin(\mt{X} \cdot \diag(1/\lambda_1,\ldots,1/\lambda_{k_\gamma},0,\ldots,0) \cdot \mt{X}^T)\geq m] \geq \delta.
\]
Since $\forall i\in[k_\gamma],\lambda_i \geq \gamma^2$, we have $\P[\lambdamin(\mt{X} \mt{X}^T)\geq m\gamma^2] \geq \delta$. 
\item Case II:
Assume $\lambda_{k_\gamma} < \gamma^2$. Then $\lambda_i < \gamma^2$ for all $i \in \{k_\gamma,\ldots,d\}$. 
Consider the random variables $Z[i] = X[i]/\gamma$ for $i \in \{k_\gamma,\ldots,d\}$. $Z[i]$ is sub-Gaussian with relative moment $\rmom$
and $\E[Z[i]^2] \leq 1$.
Consider the product distribution of $Z[k_\gamma],\ldots,Z[d]$,
and let $\Sigma'$ be its covariance matrix. We have $\Sigma' < I_{d-k_\gamma+1}$.
By the minimality in \eqref{eq:kgammamin} we also have $\trace(\Sigma') = \frac{1}{\gamma^2}\sum_{i=k_\gamma}^d \lambda_i \geq k_\gamma-1$.
Let $\mt{Z}$ be
the matrix of a sample of size $m$ drawn from this product distribution. By
 \thmref{thm:smallesteigwhpsg},
$\P[\lambdamin(\mt{Z} \mt{Z}^T)\geq m] \geq \delta$. Equivalently,
\[
\P[\lambdamin(\mt{X} \cdot \diag(0,\ldots,0,1/\gamma^2,\ldots,1/\gamma^2) \cdot \mt{X}^T)\geq m] \geq \delta,
\]
therefore $\P[\lambdamin(\mt{X} \mt{X}^T)\geq m\gamma^2] \geq \delta$.
\end{itemize}

In both cases $\P[\lambdamin(\mt{X}\mt{X}^T)\geq m\gamma^2] \geq \delta$. This holds for any $m \leq \beta (k_\gamma-1) -C$, thus by \thmref{thm:inductive} $m(\epsilon, \gamma,D, \delta/2) \geq \floor{(\beta(k_\gamma-1)-C)/2}$
for $\epsilon < 1/2-\loss_\gamma^*(D)$.
We finalize the proof by setting $\delta = \frac{1}{2}$ and adjusting $\beta$ and $C$.
\end{proof}

\section{On the Limitations of the Covariance Matrix}\label{sec:limitations}
We have shown matching upper and lower bounds for the sample complexity of
learning with MEM, for any sub-Gaussian product
distribution with a bounded relative moment. This shows that the
margin-adapted dimension fully characterizes the sample complexity of learning
with MEM for such distributions. What properties of a
distribution play a role in determining the sample complexity for general distributions?
In the following theorem we show
that these properties must include more than the covariance matrix of the
distribution, even when assuming sub-Gaussian tails and bounded relative moments. 

\begin{theorem}\label{thm:covariance}
For any integer $d > 1$, there
exist two distributions $D$ and $P$ over $\reals^d \times \{\pm 1\}$
with identical covariance matrices, such that for any $\epsilon,\delta \in (0,\frac14)$, $m(\epsilon, 1, P, \delta) \geq \Omega(d)$ while $m(\epsilon, 1, D,\delta) \leq \ceil{\log_2(1/\delta)}$. Both $D_X$ and $P_X$ 
are sub-Gaussian random vectors, with a relative moment of $\sqrt{2}$ in all directions.
\end{theorem}

\begin{proof}
Let $D_a$ and $D_b$ be distributions over $\reals^d$ such that $D_a$ is uniform over $\{\pm 1\}^d$ 
and $D_b$ is uniform over $\{\pm 1\}\times\{0\}^{d-1}$. Let $D_X$ be a balanced mixture of $D_a$ and $D_b$. Let $P_X$ be uniform over $\{\pm 1\}\times\{\frac{1}{\sqrt{2}}\}^{d-1}$.
For both $D$ and $P$,  let $\P[Y = \dotprod{e_1, X}] = 1$.  
The covariance matrix of $D_X$ and $P_X$ is $\diag(1,\half, \ldots, \half)$,  thus $k_1(D_X) = k_1(P_X) \geq \Omega(d)$.

By \eqref{eq:bernoulli}, $P_X,D_a$ and $D_b$ are all sub-Gaussian product distribution with relative moment $1$, thus also with moment $\sqrt{2} > 1$.
The projection of $D_X$ along any direction $u \in \reals^d$ is sub-Gaussian with relative moment $\sqrt{2}$ as well, since
\begin{align*}
&\E_{X \sim D_X}[\exp(\dotprod{u,X})] = \half(\E_{X \sim D^a}[\exp(\dotprod{u,X})] + \E_{X \sim D^b}[\exp(\dotprod{u,X})]) \\
&=\half(\prod_{i\in[d]}(\exp(u_i)+\exp(-u_i))/2 + (\exp(u_1)+\exp(-u_1))/2) \\
&\leq \half (\prod_{i\in[d]}\exp(u_i^2/2) + \exp(u_1^2/2))\leq \exp(\norm{u}^2/2) \leq  \exp((\norm{u}^2+u_1^2)/2)\\
&=\exp(\E_{X\sim D_X} [\dotprod{u, X}^2]).
\end{align*}
For $P$ we have by \thmref{thm:lowerboundsg} that for any $\epsilon \leq \frac14$, $m(\epsilon, 1, P,\frac14) \geq \Omega(k_1(P_X)) \geq \Omega(d)$. 
In contrast, any MEM algorithm $\cA_1$ will output the correct separator for $D$ 
whenever the sample has at least one point drawn from $D_b$. This is because the separator $e_1$ is the only $w\in \ball$ that classifies this point with zero $1$-margin errors. Such a point exists in a sample of size $m$ with probability $1-2^{-m}$. Therefore $\loss_0(\cA_1,D,m,1/2^m) = 0$.
It follows that for all $\epsilon > 0$, $m(\epsilon,1,D,\delta) \leq \ceil{\log_2(1/\delta)}$.
\end{proof}

\section{Conclusions}\label{sec:conclusions}
\corref{cor:upperbound} and \thmref{thm:lowerboundsg} together provide a tight characterization of the sample complexity of any sub-Gaussian product distribution with a bounded relative moment. Formally, fix $\rmom > 0$. For any $D$ such that $D_X \in \dfamily_\rmom$, and for any $\gamma > 0$ and $\epsilon \in (0,\frac{1}{2} - \loss^*_\gamma(D))$
\begin{equation}\label{eq:doublebound}
 \Omega(k_\gamma(D_X)) \leq m(\epsilon,\gamma,D) \leq \tilde{O}\left(\frac{k_{\gamma}(D_X)}{\epsilon^2}\right).
 \end{equation}
 The upper bound holds uniformly for all distributions, and the constants in the lower bound depend only on $\rmom$. This result shows that the true sample complexity of learning each of these distributions with MEM is characterized by the \fullkgname. 
An interesting conclusion can be drawn as to the influence of the conditional distribution of labels $D_{Y|X}$: Since \eqref{eq:doublebound} holds for \emph{any} $D_{Y|X}$, the effect of the direction of the best separator on the sample complexity is bounded, even for highly non-spherical distributions.

We note that the upper bound that we have proved involves logarithmic factors which might not be necessary. There are upper bounds that depend on the margin alone and on the dimension alone without logarithmic factors. On the other hand, in our bound, which combines the two quantities, there is a logarithmic dependence which stems from the margin component of the bound. It might be possible to tighten the bound and remove the logarithmic dependence.

\eqref{eq:doublebound} can be used to easily characterize the sample complexity behavior for interesting distributions, to compare $L_2$ margin minimization to other learning methods, and to improve certain active learning strategies. We elaborate on each of these applications in the following examples.

\begin{example}[Gaps between $L_1$ and $L_2$ regularization in the presence of
  irrelevant features] \hspace{0.3in} \linebreak[4]
\citet{Ng04} considers learning a single relevant
feature in the presence of many irrelevant features, and compares
using $L_1$ regularization and $L_2$ regularization.  When $\norm{X}_{\infty} \leq 1$,
upper bounds on learning with $L_1$ regularization guarantee a sample
complexity of $O(\ln(d))$ for an $L_1$-based learning rule
\citep{Zhang02}.  In order to compare this with the sample complexity of
$L_2$ regularized learning and establish a gap, one must use a {\em
  lower bound} on the $L_2$ sample complexity.  The argument provided by Ng
actually assumes scale-invariance of the learning rule, and is
therefore valid only for {\em unregularized} linear learning.  In contrast,
using our results we can easily establish a lower
bound of $\Omega(d)$ for many specific distributions with
a bounded $\norm{X}_{\infty}$ and $Y=\sign(X[i])$ for some $i$.  
For instance, if each coordinate is a bounded independent sub-Gaussian random variable with a bounded relative moment, we have $k_1 = \ceil{d/2}$ and \thmref{thm:lowerboundsg}
implies a lower bound of $\Omega(d)$ on the $L_2$ sample complexity.
\end{example}

\begin{example}[Gaps between generative and discriminative learning for a Gaussian mixture]
Let there be two classes, each drawn from a unit-variance spherical
Gaussian in $\reals^d$ with a large distance $2v
>> 1$ between the class means, such that $d >> v^4$. Then $\P_D[X|Y=y]
= \cN(y v\cdot e_1,I_d)$, where $e_1$ is a unit vector in
$\reals^d$.  For any $v$ and $d$, we have $D_X \in \dfamily_1$.  For
large values of $v$, we have extremely low margin error at $\gamma =
v/2$, and so we can hope to learn the classes by looking for a
large-margin separator.  Indeed, we can calculate $k_\gamma =
\ceil{d/(1+\frac{v^2}{4})}$, and conclude that the required sample complexity is $\tilde{\Theta}(d/v^2)$.  Now consider a generative
approach: fitting a spherical Gaussian model for each class. This
amounts to estimating each class center as the empirical average of
the points in the class, and classifying based on the nearest
estimated class center.  It is possible to show that for any constant
$\epsilon>0$, and for large enough $v$ and $d$, $O(d/v^4)$ samples are
enough in order to ensure an error of $\epsilon$.  This establishes a
rather large gap of $\Omega(v^2)$ between the sample complexity of the
discriminative approach and that of the generative one. 
\end{example}

\begin{example}[Active learning] In active learning, there is an abundance of unlabeled examples, but 
labels are costly, and the active learning algorithm needs to decide which labels to query based 
on the labels seen so far. A popular approach to active learning involves estimating
the current set of possible classifiers using sample complexity upper bounds \citep[see, e.g.,][]{BalcanBeLa09,BeygelzimerHsLaZh10b}. Without any distribution-specific information, only general distribution-free upper bounds can be used. However, since there is an abundance of unlabeled examples,
the active learner can use these to estimate tighter distribution-specific upper bounds. In the case of linear classifiers, the \fullkgname\ can be calculated from the uncentered covariance matrix of the distribution, which can be easily estimated from unlabeled data. Thus, our sample complexity upper bounds can be used to improve the active learner's label complexity. Moreover, the lower bound suggests that any further improvement of such active learning strategies would require more information other than the 
distribution's covariance matrix. 
\end{example}

To summarize, we have shown that the true sample complexity of large-margin
learning of each of a rich family of distributions is characterized by the
\fullkgname. Characterizing the true sample complexity allows a better
comparison between this learning approach and other algorithms, and has many
potential applications. The challenge of
characterizing the true sample complexity extends to any distribution and any
learning approach. \thmref{thm:covariance} shows that other properties but the covariance matrix must be taken into account for general distributions.
We believe that obtaining answers to these questions is of
great importance, both to learning theory and to learning applications.

\acks{The authors thank Boaz Nadler for many insightful discussions. 
During part of this research, Sivan Sabato was supported by the Adams Fellowship Program of the Israel Academy of Sciences and Humanities. This work is partly supported by the Gatsby Charitable Foundation, The DARPA MSEE project, the Intel ICRI-CI center, and the Israel Science Foundation center of excellence grant.}

\appendix
\section{Proofs Omitted from the Text}\label{app:proofs}

In this appendix we give detailed proofs which were omitted from the text.
\subsection{Proof of Proposition~\ref{prop:ramp}} \label{app:radramp}
\begin{proof}
Let $w^* \in \argmin_{w\in \ball} \loss_\gamma(w,D)$.
By \eqref{eq:rademacherind}, with probability $1-\delta/2$
\[
\ramp_\gamma(\cA_\gamma(S), D) \leq \ramp_\gamma(\cA_\gamma(S), S) + 2\cR_m(\rampf_\gamma, D) + \sqrt{\frac{8\ln(2/\delta)}{m}}.
\]

Set $h^* \in \cH$ such that $\loss_\gamma(h^*, D) = \loss_\gamma^*(\cH,D)$.
We have
\[
\ramp_\gamma(\cA_\gamma(S), S) \leq \loss_\gamma(\cA_\gamma(S),S) \leq \loss_\gamma(h^*,S).
\]
The first inequality follows since the ramp loss is upper bounded by the margin loss. The second inequality follows since $\cA$ is a MEM algorithm.
Now, by Hoeffding's inequality, since the range of $\ramp_\gamma$ is in $[0,1]$, with probability at least $1-\delta/2$
\[
\loss_\gamma(h^*,S) \leq \loss_\gamma(h^*, D) + \sqrt{\frac{\ln(2/\delta)}{2m}}.
\]
It follows that with probability $1-\delta$
\begin{equation}\label{eq:ramp1}
\ramp_\gamma(\cA_\gamma(S), D) \leq \loss_\gamma^*(\cH,D) + 2\cR_m(\rampf_\gamma, D) + \sqrt{\frac{14\ln(2/\delta)}{m}}.
\end{equation}
We have $\loss_0 \leq \ramp_\gamma$. Combining this with \eqref{eq:ramp1} we conclude 
\eqref{eq:propstatement}.
\end{proof}

\subsection{Proof of \lemref{lem:glipschitz}}\label{app:glipschitz}
\begin{proof}[of \lemref{lem:glipschitz}]
For a function $f:\cX\rightarrow \reals$ and a $z \in Z$, define the function $G[f,z]$ by
\[
\forall x \in \cX,\quad G[f,z](x) = \chop{f(x) + z(x)} - f(x).
\]
Let $f_1,f_2 \in \reals^\cX$ be two functions, and let $g_1 = G[f_1,z] \in \cG(f_1)$ for some $w_b \in \bar{V}$.
Then, since $G[f_2,z] \in \cG(f_2)$, we have
\[
\inf_{g_2 \in \cG(f_2)}\norm{g_1 -g_2}_{L_2(S)} \leq \norm{G[f_1,z] - G[f_2,z]}.
\]
Now, for all $x\in\reals$,
\begin{align*}
|G[f_1,z](x) - G[f_2,z](x)| &=|\chop{f_1(x) + z(x)} - f_1(x)
- \chop{f_2(x) + z(x)} + f_2(x)| \\
&\leq |f_1(x) - f_2(x)|.
\end{align*}
Thus, for any $S \subseteq \cX$,
\begin{align*}
\norm{G[f_1,z] - G[f_2,z]}^2_{L_2(S)} &= \E_{X \sim S}(G[f_1,z](X) - G[f_2,z](X))^2\\ 
&\leq \E_{X \sim S}(f_1(X) - f_2(X))^2 = \norm{f_1 -f_2}^2_{L_2(S)}.
\end{align*}
It follows that $\inf_{g_2 \in \cG(f_2)}\norm{g_1 -g_2}_{L_2(S)} \leq \norm{f_1 -f_2}_{L_2(S)}$. This holds for any $g_1 \in \cG(f_1)$, thus $\Delta_H(\cG(f_1),\cG(f_2))\leq\norm{f_1 -f_2}_{L_2(S)}$.
\end{proof}

\subsection{Proof of \lemref{lem:pseudo}}\label{app:lempseudo}
\begin{proof}[of \lemref{lem:pseudo}]
Let $k$ be the pseudo-dimension of $\cG(f)$, and let $\{ x_1,\ldots,x_k\}
\subseteq \cX$ be a set which is pseudo-shattered by $\cG(f)$. We show that the
same set is pseudo-shattered by $Z$ as well, thus proving the lemma.  Since $\cG(f)$ 
is pseudo-shattered, there
exists a vector $r \in \reals^k$ such that for all $y \in \binlab^k$ there
exists a $g_y \in \cG(f)$ such that $\forall i\in [m], \sign(g_y(x_i)-r[i]) = y[i]$. Therefore for all $y
\in \binlab^k$ there exists a $z_y \in Z$ such that
\[
\forall i \in [k], \sign(\chop{f(x_i) + z_y(x_i)}-f(x_i) - r[i]) = y[i].
\] 
By considering the case $y[i] = 1$, 
we have
\[0 < \chop{f(x_i) + z_y(x_i)}-f(x_i) - r[i] \leq 1 -f(x_i) - r[i].\]
By considering the case $y[i] = -1$, 
we have
\[0 > \chop{f(x_i) + z_y(x_i)}-f(x_i) - r[i] \geq -f(x_i) - r[i].\]
Therefore $0 < f(x_i) + r[i] < 1$.
Now, let $y \in \binlab^k$ and consider any $i \in [k]$.  If $y[i] = 1$ then 
\[
\chop{f(x_i) + z_y(x_i)}-f(x_i) - r[i] > 0
\]
 It follows that
\[
\chop{f(x_i) + z_y(x_i)} > f(x_i) + r[i] > 0,
\]
 thus 
\[
f(x_i) + z_y(x_i) > f(x_i) + r[i].
\]
In other words, $\sign(z_y(x_i)-r[i]) = 1 = y[i]$.
If $y[i] = -1$ then
\[
\chop{f(x_i) + z_y(x_i)}-f(x_i) - r[i] < 0.
\]
It follows that
\[
\chop{f(x_i) + z_y(x_i)} < f(x_i) + r[i] < 1,
\]
thus 
\[
f(x_i) + z_y(x_i) < f(x_i) + r[i].
\]
in other words, $\sign(z_y(x_i)-r[i]) = -1 = y[i]$.
We conclude that $Z$ shatters $\{ x_1,\ldots,x_k\}$ as well, using the same
vector $r \in \reals^k$. Thus the pseudo-dimension of $Z$ is at least $k$.
\end{proof}

\subsection{Proof of \lemref{lem:exactmargin}}\label{app:shattercond}

To prove \lemref{lem:exactmargin}, we first prove the following lemma. Denote by $\conv(A)$ the convex hull of a set $A$.
\begin{lemma}\label{lem:conv} 
Let $\gamma > 0$. For each $y \in \binm$, select $r_y \in \reals^m$ such that for all $i \in [m]$, $r_y[i]y[i] \geq \gamma$. Let $R = \{r_y \in \reals^m \mid y \in \binm\}$. Then $\{\pm \gamma\}^m \subseteq \conv(R)$. 
\end{lemma}
\begin{proof}
We will prove the claim by induction on the dimension $m$. 
For the base case, if $m=1$, we have $R = \{a,b\}\subseteq \reals$ where $a \leq -\gamma$ and $b \geq \gamma$. Clearly, $\conv(R) = [a,b]$, and $\pm \gamma \in [a,b]$.

For the inductive step, assume the lemma holds for $m-1$.
For a vector $t \in \reals^m$, denote by $\bar{t}$ its projection $(t[1],\ldots, t[m-1])$ on $\reals^{m-1}$. Similarly, for a set of vectors $S \subseteq \reals^m$, let \linebreak[4] \mbox{$\bar{S} = \{\bar{s} \mid s \in S\} \subseteq \reals^{m-1}$}. 
 Define $Y_+ = \binlab^{m-1} \times \{+1\}$ and $Y_- = \binlab^{m-1} \times \{-1\}$.
Let $R_+ = \{r_y \mid y\in Y_+ \}$, and similarly for $R_-$. Then the induction hypothesis holds for $\bar{R}_+$ and $\bar{R}_-$ with dimension $m-1$.
Let $z \in \{\pm \gamma\}^m$. We wish to prove $z \in \conv(R)$.
From the induction hypothesis we have $\bar{z} \in \conv(\bar{R}_+)$ and $\bar{z} \in \conv(\bar{R}_-)$. Thus, for all $y\in \{\pm1\}$ there exist 
$
\alpha_y,\beta_y \geq 0$ such that $\sum_{y\in Y_+} \alpha_y = \sum_{y\in Y_-} \beta_y =1$, and 
\[
\bar{z} = \sum_{y\in Y_+} \alpha_y \bar{r}_y = \sum_{y\in Y_-} \beta_y \bar{r}_y.
\]
Let $z_a = \sum_{y\in Y_+} \alpha_y r_y$ and $z_b = \sum_{y\in Y_-} \beta_y r_y$
We have that $\forall y \in Y_+, r_y[m] \geq \gamma$, and $\forall y \in Y_-,
r_y[m] \leq -\gamma$. Therefore, $z_b[m] \leq -\gamma \leq z[m] \leq \gamma \leq z_a[m].$
In addition, $\bar{z}_a = \bar{z}_b = \bar{z}$. Select $\lambda \in [0,1]$
 such that $z[m] = \lambda z_a[m] + (1-\lambda) z_b[m]$, then 
 $z = \lambda z_a + (1-\lambda) z_b$. 
Since $z_a,z_b \in \conv(R)$, we have $z \in \conv(R)$.
\end{proof}

\begin{proof}[of \lemref{lem:exactmargin}]
Denote by $f(S)$ the vector $(f(x_1), \ldots, f(x_m))$.
Recall that $r \in \reals^m$ is the witness for the shattering of $S$, and let
\[
L = \{f(S) - r \mid f \in \cF\} \subseteq \reals^m.
\]
 Since $S$ is shattered, for any $y \in
\binm$ there is an $r_y \in L$ such that $\forall i\in [m], r_y[i]y[i] \geq
\gamma$. By \lemref{lem:conv}, $\{\pm \gamma\}^m \subseteq \conv(L)$. 
Since $\cF$ is convex, $L$ is also convex. Therefore $\{\pm \gamma\}^m \subseteq L$.
\end{proof}

\subsection{Proof of Lemma \ref{lem:sgvecmgf}}\label{app:sgvecmgf}

\begin{proof}[of \lemref{lem:sgvecmgf}]
It suffices to consider diagonal moment matrices:
If $\mt{B}$ is not diagonal, let $\mt{V} \in \reals^{d\times d}$ be an orthogonal matrix such that $\mt{V} \mt{B} \mt{V}^T$ is diagonal, and let $Y = \mt{V}X$. We have $\E[\exp(t \norm{Y}^2)] = \E[\exp(t \norm{X}^2)]$ and $\trace(\mt{V}\mt{B}\mt{V}^T) = \trace(\mt{B})$. In addition, for all $u \in \reals^d$, 
\begin{align*}
\E[\exp(\dotprod{u,Y})] &= \E[\exp(\dotprod{\mt{V}^T u, X})]\leq \exp(\half \dotprod{\mt{B} \mt{V}^T u, \mt{V}^T u}) = \exp(\half \dotprod{\mt{V} \mt{B} \mt{V}^T u, u}).
\end{align*}
Therefore $Y$ is sub-Gaussian with the diagonal moment matrix $\mt{V}\mt{B}\mt{V}^T$.
Thus assume w.l.o.g. that $\mt{B}  = \diag(\lambda_1,\ldots,\lambda_d)$ where $\lambda_1 \geq \ldots \geq \lambda_d \geq 0$.

We have $\exp(t\norm{X}^2) = \prod_{i\in[d]}\exp(tX[i]^2)$. In addition, for any $t > 0$ and $x\in \reals$,
$
2\sqrt{\Pi t}\cdot\exp(t x^2) = \int_{-\infty}^\infty \exp(s x-\frac{s^2}{4 t})ds.
$
Therefore, for any $u\in\reals^d$,
\begin{align*}
(2\sqrt{\Pi t})^d \cdot \E[\exp(t\norm{X}^2)] &= \E\left[\prod_{i\in[d]} \int_{-\infty}^\infty \exp(u[i] X[i]-\frac{u[i]^2}{4 t})du[i]\right]\\
&= \E\left[\int_{-\infty}^{\infty} \ldots
  \int_{-\infty}^{\infty} \prod_{i\in[d]} \exp(u[i] X[i]-\frac{u[i]^2}{4 t})du[i]\right]\\
&= \E\left[\int_{-\infty}^{\infty} \ldots
  \int_{-\infty}^{\infty} \exp(\dotprod{u,X}-\frac{\norm{u}^2}{4 t})\prod_{i\in[d]} du[i]\right]\\
&=  \int_{-\infty}^{\infty} \ldots
  \int_{-\infty}^{\infty} \E[\exp(\dotprod{u, X})]\exp(-\frac{\norm{u}^2}{4 t})\prod_{i\in[d]} du[i]
\end{align*}
By the sub-Gaussianity of $X$, the last expression is bounded by
{\allowdisplaybreaks
\begin{align*}
&\leq\int_{-\infty}^{\infty} \ldots \int_{-\infty}^{\infty} \exp(\half \dotprod{\mathbf{\mt{B}}u,u}-\frac{\norm{u}^2}{4 t})\prod_{i\in[d]} du[i]\\ 
&= \int_{-\infty}^{\infty} \ldots \int_{-\infty}^{\infty} \prod_{i\in[d]} \exp( \frac{\lambda_iu[i]^2}{2} - \frac{u[i]^2}{4t})  du[i]\\
&= \prod_{i\in[d]} \int_{-\infty}^{\infty} \exp(u[i]^2(\frac{\lambda_i}{2}-\frac{1}{4t}))du[i] = \Pi^{d/2}\big(\prod_{i\in[d]}(\frac{1}{4t}-\frac{\lambda_i}{2})\big)^{-\half}.
\end{align*}
}
The last equality follows from the fact that for any $a > 0$, $\int_{-\infty}^{\infty} \exp(-a \cdot s^2)ds= \sqrt{\Pi/a}$,
and from the assumption $t \leq  \frac{1}{4\lambda_1}$.
We conclude that 
\[
\E[\exp(t \norm{X}^2)] \leq (\prod_{i\in[d]}(1-2\lambda_i t))^{-\half} \leq \exp(2t \cdot \sum_{i=1}^d \lambda_i) = \exp(2t\cdot\trace(\mt{B})),
\]
where the second inequality holds since  $\forall x \in[0,1]$, $(1-x/2)^{-1}\leq \exp(x)$. 
\end{proof}

\subsection{Proof of Theorem \ref{thm:smallesteigwhpsg}}\label{app:smallesteigwhpsg}

In the proof of \thmref{thm:smallesteigwhpsg} we use the fact   
$\lambdamin(\mt{X}\mt{X}^T) = \inf_{\norm{x}_2=1}\norm{\mt{X}^Tx}^2$ and bound the right-hand side
via an $\epsilon$-net of the unit sphere in $\reals^m$, denoted by $S^{m-1} \triangleq \{ x \in \reals^m \mid \norm{x}_2 = 1\}$. 
An $\epsilon$-net of the unit sphere is a set $C \subseteq S^{m-1}$ such that 
$\forall x \in S^{m-1}, \exists x' \in C, \norm{x-x'}\leq \epsilon$. Denote the minimal size of an $\epsilon$-net for $S^{m-1}$ by $\cN_m(\epsilon)$, and by $\cC_m(\epsilon)$ a minimal $\epsilon$-net of $S^{m-1}$, so that $\cC_m(\epsilon) \subseteq S^{m-1}$ and $|\cC_m(\epsilon)| = \cN_m(\epsilon)$.
The proof of \thmref{thm:smallesteigwhpsg}
requires several lemmas.
First we prove a concentration result for the norm of a matrix defined by sub-Gaussian variables.
Then we bound the probability that the squared norm of a vector is small.
\begin{lemma}\label{lem:boundnormsubg}
Let $\mt{Y}$ be a $d\times m$ matrix with $m \leq d$, such that $\mt{Y}_{ij}$ are independent sub-Gaussian variables with moment $B$.
Let $\Sigma$ be a diagonal $d\times d$ PSD matrix such that $\Sigma \leq I$. Then for all $t \geq 0$ and $\epsilon \in (0,1)$,
\[
\P[\norm{\sqrt{\Sigma}\mt{Y}} \geq t] \leq \cN_m(\epsilon)\exp(\frac{\trace(\Sigma)}{2}-\frac{t^2(1-\epsilon)^2}{4B^2}).
\]
\end{lemma}
\begin{proof}
We have $\norm{\sqrt{\Sigma}\mt{Y}} \leq \max_{x \in \cC_m(\epsilon)}\norm{\sqrt{\Sigma}\mt{Y}x}/(1-\epsilon)$, see for instance in \cite{Bennet75}.
Therefore, 
\begin{equation}\label{eq:mat2vec}
\P[\norm{\sqrt{\Sigma}\mt{Y}} \geq t] \leq \sum_{x \in \cC_m(\epsilon)}\P[\norm{\sqrt{\Sigma}\mt{Y}x} \geq (1-\epsilon)t].
\end{equation}
Fix $x \in \cC_m(\epsilon)$. Let $V = \sqrt{\Sigma}\mt{Y}x$, and assume $\Sigma = \diag(\lambda_1,\ldots,\lambda_d)$. For $u \in \reals^d$,
\begin{align*}
&\E[\exp(\dotprod{u, V})] = \E[\exp(\sum_{i\in[d]}u_i\sqrt{\lambda}_i \sum_{j\in[m]}\mt{Y}_{ij}x_j)]
= \prod_{j,i}\E[\exp(u_i\sqrt{\lambda}_i \mt{Y}_{ij}x_j)]\\
&\quad\leq \prod_{j,i}\exp(u_i^2\lambda_i B^2 x_j^2/2) = \exp(\frac{B^2}{2}\sum_{i\in[d]}u_i^2\lambda_i \sum_{j\in[m]}x_j^2 )\\
&\quad= \exp(\frac{B^2}{2} \sum_{i\in[d]}u_i^2\lambda_i) = \exp(\dotprod{B^2\Sigma u,u}/2).
\end{align*}
Thus $V$ is a sub-Gaussian vector with moment matrix $B^2\Sigma$.
Let $s = 1/(4B^2)$. Since $\Sigma \leq I$, we have $s \leq 1/(4B^2\max_{i\in[d]}\lambda_i)$.
Therefore, by \lemref{lem:sgvecmgf}, 
\[
\E[\exp(s\norm{V}^2)] \leq \exp(2sB^2\trace(\Sigma)).
\]
By Chernoff's method, 
$\P[\norm{V}^2 \geq z^2] \leq \E[\exp(s\norm{V}^2)]/\exp(sz^2)$. Thus
\[
\P[\norm{V}^2 \geq z^2] \leq \exp(2sB^2\trace(\Sigma) - sz^2) = \exp(\frac{\trace(\Sigma)}{2}-\frac{z^2}{4B^2}).
\]
Set $z = t(1-\epsilon)$. Then for all $x \in S^{m-1}$
\[
\P[\norm{\sqrt{\Sigma}\mt{Y}x} \geq t(1-\epsilon)] = \P[\norm{V} \geq t(1-\epsilon)] \leq \exp(\frac{\trace(\Sigma)}{2}-\frac{t^2(1-\epsilon)^2}{4B^2}).
\]
Therefore, by \eqref{eq:mat2vec},
\[
\P[\norm{\sqrt{\Sigma}\mt{Y}} \geq t] \leq \cN_m(\epsilon)\exp(\frac{\trace(\Sigma)}{2}-\frac{t^2(1-\epsilon)^2}{4B^2}).
\]
\end{proof}

\begin{lemma}\label{lem:boundsigmayxsubg}
Let $\mt{Y}$ be a $d\times m$ matrix with $m \leq d$, such that $\mt{Y}_{ij}$ are independent 
centered random variables with variance $1$ and fourth moments at most $B$.
Let $\Sigma$ be a diagonal $d\times d$ PSD matrix such that $\Sigma \leq I$.
There exist $\alpha > 0$ and $\eta \in (0,1)$ that depend only on $B$ such that for any $x \in S^{m-1}$
\[
\P[\norm{\sqrt{\Sigma}\mt{Y}x}^2\leq \alpha\cdot (\trace(\Sigma)-1)] \leq \eta^{\trace(\Sigma)}.
\]
\end{lemma}

To prove \lemref{lem:boundsigmayxsubg} we require \lemref{lem:sumofbounded} \citep[Lemma 2.2]{RudelsonVe08} and \lemref{lem:sigyxboundbelow}, which extends Lemma 2.6 in the same work.
\begin{lemma}\label{lem:sumofbounded}
Let $T_1,\ldots,T_n$ be independent 
non-negative random variables.
Assume that there are $\theta > 0$ and $\mu \in (0,1)$ such that for any $i$, $\P[T_i \leq \theta] \leq \mu$.
There are $\alpha > 0$ and $\eta \in (0,1)$ that depend only on $\theta$ and $\mu$ such that 
\[
\P[\sum_{i=1}^n T_i < \alpha n] \leq \eta^n.
\]
\end{lemma}

\begin{lemma}\label{lem:sigyxboundbelow}
Let $\mt{Y}$ be a $d\times m$ matrix with $m \leq d$, such that the columns of $\mt{Y}$ are i.i.d.~random vectors. Assume further that $\mt{Y}_{ij}$ are centered, and have a variance of $1$ and a fourth moment at most $B$.
Let $\Sigma$ be a diagonal $d \times d$ PSD matrix.  
Then for all $x \in S^{m-1}$,
\[
\P[\norm{\sqrt{\Sigma} \mt{Y} x} \leq \sqrt{\trace(\Sigma)/2}] \leq 1-1/(196B).
\]
\end{lemma}
\begin{proof}
Let $x\in S^{m-1}$, and $T_i = (\sum_{j=1}^m \mt{Y}_{ij} x_j)^2$. Let $\lambda_1,\ldots,\lambda_d$ be the values on the diagonal of $\Sigma$, and let $T_\Sigma = \norm{\sqrt{\Sigma} \mt{Y} x}^2 = \sum_{i=1}^d \lambda_i T_i$. 
First, since $\E[\mt{Y}_{ij}] = 0$ and $\E[\mt{Y}_{ij}] = 1$ for all $i,j$, we have 
\[
\E[T_i] = \sum_{i\in [m]} x^2_j\E[\mt{Y}_{ij}^2] = \norm{x}^2 = 1.
\]
Therefore $\E[T_\Sigma] = \trace(\Sigma)$.
Second, since $\mt{Y}_{i1},\ldots,\mt{Y}_{im}$ are independent and centered, 
we have \citep[Lemma 6.3]{LedouxTa91}
\[
\E[T_i^2] = \E[(\sum_{j\in[m]} \mt{Y}_{ij} x_j)^4] \leq 16\E_\sigma[(\sum_{j\in[m]} \sigma_j\mt{Y}_{ij} x_j)^4],
\]
where $\sigma_1,\ldots,\sigma_m$ are independent uniform $\binlab$ variables.
Now, by Khinchine's inequality \citep{NazarovPo00},
\[
\E_\sigma[(\sum_{j\in[m]} \sigma_j\mt{Y}_{ij} x_j)^4] \leq 3\E[(\sum_{j\in[m]} \mt{Y}^2_{ij} x^2_j)^2]
= 3\sum_{j,k\in[m]} x^2_j x^2_k\E[\mt{Y}^2_{ij}]\E[\mt{Y}^2_{ik}].
\]
Now $\E[\mt{Y}^2_{ij}]\E[\mt{Y}^2_{ik}] \leq \sqrt{\E[\mt{Y}^4_{ij}]\E[\mt{Y}^4_{ik}]} \leq B$.
Thus $\E[T_i^2] \leq 48B\sum_{j,k\in[m]} x^2_j x^2_k = 48B\norm{x}^4 = 48B.$
Thus, 
\begin{align*}
\E[T_\Sigma^2] &= \E[(\sum_{i=1}^d \lambda_i T_i)^2] =  \sum_{i,j=1}^d \lambda_i \lambda_j \E[ T_i T_j] \\
&\leq \sum_{i,j=1}^d \lambda_i \lambda_j \sqrt{\E[T_i^2] \E[T_j^2]} \leq 48B (\sum_{i=1}^d \lambda_i)^2 = 48B\cdot\trace(\Sigma)^2.
\end{align*}
By the Paley-Zigmund inequality \citep{PaleyZy32}, for $\theta \in [0,1]$
\[
\P[T_\Sigma \geq \theta \E[T_\Sigma]] \geq (1-\theta)^2 \frac{\E[T_\Sigma]^2}{\E[T_\Sigma^2]} \geq \frac{(1-\theta)^2}{48B}.
\]
Therefore, setting $\theta = 1/2$, we get $\P[T_\Sigma \leq \trace(\Sigma)/2] \leq 1 - 1/(196B).$
\end{proof}

\begin{proof}[of \lemref{lem:boundsigmayxsubg}]
Let $\lambda_1,\ldots,\lambda_d \in [0,1]$ be the values on the diagonal of $\Sigma$.
Consider a partition $Z_1,\ldots,Z_k$ of $[d]$, and denote $L_j =\sum_{i\in Z_j} \lambda_i$. 
There exists such a partition such that for all $j\in [k]$, $L_j \leq 1$, and for all $j \in [k-1]$, $L_j > \half$. 
Let $\Sigma[j]$ be the sub-matrix of $\Sigma$ that includes the rows and columns whose indexes are in $Z_j$. Let $\mt{Y}[j]$ be the sub-matrix of $\mt{Y}$ that includes the rows in $Z_j$. Denote $T_j = \norm{\sqrt{\Sigma[j]} \mt{Y}[j] x}^2$.
Then
\[
\norm{\sqrt{\Sigma}\mt{Y}x}^2 = 
\sum_{j\in[k]} \sum_{i\in Z_j} \lambda_i (\sum_{j=1}^m \mt{Y}_{ij} x_j)^2 = 
\sum_{j\in[k]} T_j.
\]

We have $\trace(\Sigma) = \sum_{i=1}^d \lambda_i \geq \sum_{j\in[k-1]} L_j \geq \half(k-1)$. In addition, $L_j \leq 1$ for all $j \in [k]$. Thus
$
\trace(\Sigma) \leq k \leq 2\trace(\Sigma)+1. 
$
For all $j \in [k-1]$, $L_j \geq \half$, thus by \lemref{lem:sigyxboundbelow}, $\P[T_j \leq 1/4] \leq 1 - 1/(196B)$.
Therefore, by \lemref{lem:sumofbounded} there are $\alpha > 0$ and $\eta \in (0,1)$ that depend only on $B$ such that 
\begin{align*}
&\P[\norm{\sqrt{\Sigma}\mt{Y}x}^2 < \alpha \cdot(\trace(\Sigma)-1)] \leq 
\P[\norm{\sqrt{\Sigma}\mt{Y}x}^2 < \alpha (k-1)] \\
&\quad= \P[\sum_{j\in[k]} T_j < \alpha (k-1)] \leq
 \P[\sum_{j\in[k-1]} T_j < \alpha (k-1)] \leq \eta^{k-1} \leq \eta^{2\trace(\Sigma)}.
\end{align*}
The lemma follows by substituting $\eta$ for $\eta^2$.
\end{proof}

\begin{proof}[of \thmref{thm:smallesteigwhpsg}]
We have
\begin{align}\label{eq:lambdam}
&\sqrt{\lambdamin(\mt{X}\mt{X}^T)} = \inf_{x\in S^{m-1}}\norm{\mt{X}^Tx} \geq 
\min_{x \in \cC_m(\epsilon)}\norm{\mt{X}^Tx}-\epsilon\norm{\mt{X}^T}.
\end{align}
For brevity, denote $L = \trace(\Sigma)$. Assume $L \geq 2$.
Let $m\leq L\cdot\min(1,(c-K\epsilon)^2)$ where $c, K, \epsilon$ are constants that will be set later such that $c- K\epsilon > 0$. By \eqref{eq:lambdam}
\begin{align}
&\P[\lambdamin(\mt{X}\mt{X}^T) \leq m] \leq
\P[\lambdamin(\mt{X}\mt{X}^T) \leq (c-K\epsilon)^2L] \notag\\
&\quad\leq 
  \P[\min_{x\in\cC_m(\epsilon)}\norm{\mt{X}^Tx} -
    \epsilon\norm{\mt{X}^T} \leq (c-K\epsilon)\sqrt{L}] \label{eq:proballprev}
    \\ &\quad\leq
  \P[\norm{\mt{X}^T} \geq K\sqrt{L}] + \P[\min_{x\in\cC_m(\epsilon)}\norm{\mt{X}^Tx} \leq c\sqrt{L}].\label{eq:proball}
\end{align}
The last inequality holds since the inequality in line (\ref{eq:proballprev}) implies at least one of the inequalities in line (\ref{eq:proball}). We will now upper-bound each of the terms in line (\ref{eq:proball}).
We assume w.l.o.g. that $\Sigma$ is not singular (since zero rows and columns can be removed from $\mt{X}$ without changing $\lambdamin(\mt{X}\mt{X}^T)$). Define $\mt{Y} \triangleq \sqrt{\Sigma^{-1}}\mt{X}^T $.
Note that $\mt{Y}_{ij}$ are independent sub-Gaussian variables with (absolute) moment $\rmom$.
To bound the first term in line (\ref{eq:proball}), note that by \lemref{lem:boundnormsubg}, for any $K > 0$,
\begin{equation*}
\P[\norm{\mt{X}^T} \geq K\sqrt{L}] = \P[\norm{\sqrt{\Sigma}\mt{Y}} \geq K\sqrt{L}] \leq \cN_m(\half)\exp(L(\half-\frac{K^2}{16\rmom^2})).
\end{equation*}
By \cite{RudelsonVe09}, Proposition 2.1, for all $\epsilon \in[0,1]$, $\cN_m(\epsilon) \leq 2m(1+\frac{2}{\epsilon})^{m-1}.$ Therefore
\[
\P[\norm{\mt{X}^T} \geq K\sqrt{L}] \leq 2m5^{m-1}\exp(L(\half-\frac{K^2}{16\rmom^2})).
\]
Let $K^2 = 16\rmom^2(\frac{3}{2}+\ln(5)+\ln(2/\delta))$.
Recall that by assumption $m \leq L$, and $L \geq 2$.
Therefore
\begin{align*}
&\P[\norm{\mt{X}^T} \geq K\sqrt{L}]  \leq 2m5^{m-1}\exp(-L(1+\ln(5)+\ln(2/\delta)))\notag\\
&\quad\leq 2L5^{L-1}\exp(-L(1+\ln(5)+\ln(2/\delta))).
\end{align*}
Since $L \geq 2$, we have $2L\exp(-L) \leq 1$. Therefore
\begin{align}
\P[\norm{\mt{X}^T} \geq K\sqrt{L}]\leq 2L\exp(-L-\ln(2/\delta))\leq \exp(-\ln(2/\delta)) = \frac{\delta}{2}.\label{eq:prob1}
\end{align}

To bound the second term in line (\ref{eq:proball}), since $\mt{Y}_{ij}$ are sub-Gaussian with moment $\rmom$, $\E[\mt{Y}_{ij}^4] \leq 5\rmom^4$ \citep[Lemma 1.4]{BuldyginKo98}.
Thus, by \lemref{lem:boundsigmayxsubg}, there are $\alpha >0$ and $\eta \in (0,1)$ that depend only on $\rmom$ such that for all $x\in S^{m-1}$, $\P[\norm{\sqrt{\Sigma}\mt{Y}x}^2\leq \alpha (L-1)] \leq \eta^{L}$.
Set $c = \sqrt{\alpha/2}$. Since $L\geq 2$, we have $c\sqrt{L} \leq \sqrt{\alpha(L-1)}$. Thus
\begin{align*}
&\P[\min_{x\in\cC_m(\epsilon)}\norm{\mt{X}^Tx} \leq c\sqrt{L}] \leq
\sum_{x\in\cC_m(\epsilon)}\P[\norm{\mt{X}^Tx} \leq c \sqrt{L}] \\
&\quad\leq \sum_{x\in\cC_m(\epsilon)}\P[\norm{\sqrt{\Sigma}\mt{Y}x} \leq \sqrt{\alpha(L-1)}] 
\leq \cN_m(\epsilon)\eta^{L}.
\end{align*}

Let $\epsilon = c/(2K)$, so that $c - K\epsilon > 0$.
Let $\theta = \min(\half,\frac{\ln(1/\eta)}{2\ln(1+2/\epsilon)})$.
Set $L_\circ$ such that $\forall L \geq L_\circ$, $L \geq \frac{2\ln(2/\delta)+2\ln(L)}{\ln(1/\eta)}$. 
For $L \geq L_\circ$ and $m \leq \theta L \leq L/2$,
\begin{align}
\cN_m(\epsilon)\eta^{L} &\leq 2m(1+2/\epsilon)^{m-1}\eta^{L} \notag\\
&\leq L\exp(L(\theta \ln(1+2/\epsilon)-\ln(1/\eta)))\notag\\
&=  \exp(\ln(L) + L(\theta \ln(1+2/\epsilon)-\ln(1/\eta)/2) - L\ln(1/\eta)/2) \notag\\
&\leq  \exp(L(\theta \ln(1+2/\epsilon)-\ln(1/\eta)/2)+\ln(\delta/2))\label{eq:line1}\\
&\leq \exp(\ln(\delta/2)) = \frac{\delta}{2}.\label{eq:line3}
\end{align}
Line (\ref{eq:line1}) follows from $L \geq L_\circ$, and line (\ref{eq:line3})
follows from $\theta \ln(1+2/\epsilon)-\ln(1/\eta)/2 \leq 0$.
Set $\beta = \min\{(c-K\epsilon)^2,1,\theta\}$.
Combining \eqref{eq:proball}, \eqref{eq:prob1} and \eqref{eq:line3} we have
that if $L \geq \bar{L} \triangleq \max(L_\circ,2)$, then
$
\P[\lambdamin(\mt{X}\mt{X}^T) \leq m] \leq \delta
$
for all $m \leq \beta L$.
Specifically, this holds for all $L \geq 0$ and for all $m \leq \beta (L-\bar{L})$. Letting $C = \beta \bar{L}$ and substituting $\delta$ for $1-\delta$ we 
get the statement of the theorem.
\end{proof}

\bibliography{bib}

\end{document}